\pdfoutput=1
\documentclass[letterpaper]{article} % DO NOT CHANGE THIS
\usepackage{aaai24}  % DO NOT CHANGE THIS
\usepackage{times}  % DO NOT CHANGE THIS
\usepackage{helvet}  % DO NOT CHANGE THIS
\usepackage{courier}  % DO NOT CHANGE THIS
\usepackage[hyphens]{url}  % DO NOT CHANGE THIS
\usepackage{graphicx} % DO NOT CHANGE THIS
\usepackage{natbib}  % DO NOT CHANGE THIS AND DO NOT ADD ANY OPTIONS TO IT
\usepackage{caption} % DO NOT CHANGE THIS AND DO NOT ADD ANY OPTIONS TO IT
\usepackage{algorithm}
\usepackage{algorithmic}
\usepackage{graphicx,subcaption}
\usepackage{soul}
\usepackage{fancyhdr}
\usepackage{enumitem}
\usepackage{url}
\usepackage{color}
\usepackage{colortbl}
\usepackage[utf8]{inputenc}
\usepackage{caption}
\usepackage{makecell}
\usepackage{amsmath}
\usepackage{adjustbox}
\usepackage{amsthm}
\usepackage{amssymb}
\usepackage{booktabs}
\usepackage{algorithm}
\usepackage{tcolorbox}
\usepackage{algorithmic}
\usepackage{mathtools}
\usepackage{multirow}
\usepackage{multicol}
\usepackage{scontents}
\usepackage{lipsum}
\usepackage{tabularx}
\usepackage{cleveref}
\usepackage{adjustbox}
\usepackage{tikz}
\usepackage{array}
\usepackage{appendix}

\newtheorem{example}{Example}
\newtheorem{theorem}{Theorem}
\theoremstyle{definition}
\newtheorem{definition}{Definition}
\newcommand{\RNum}[1]{\expandafter{\romannumeral #1\relax}}
\newtheorem{lemma}{Lemma}

\usepackage{newfloat}
\usepackage{listings}
\DeclareCaptionStyle{ruled}{labelfont=normalfont,labelsep=colon,strut=off} % DO NOT CHANGE THIS
\floatstyle{ruled}
\newfloat{listing}{tb}{lst}{}
\floatname{listing}{Listing}
\title{Formal Logic Enabled Personalized Federated Learning \\ through Property Inference}
\author{
    Ziyan An, Taylor T. Johnson, Meiyi Ma\\
}
\affiliations{
    Department of Computer Science, Vanderbilt University, 
    Nashville, TN, USA \\
    \{ziyan.an, taylor.johnson, meiyi.ma\}@vanderbilt.edu
}

\usepackage{bibentry}
\begin{document}

\maketitle

\begin{abstract}
% \meiyi{one sentence on motivation}
% Recent progresses in federated learning (FL) have facilitated the development of decentralized collaborative Internet-of-Things (IoT) applications. 
% However, data-driven FL algorithms face the challenge of heterogeneity in participating IoT devices, including their deployment environment and calibration settings. 
% Fail to follow these device-specific properties can degenerate the model performance. 
% To address this issue, we present FedSTL in this work, which is a two-staged personalized FL framework with clustering for sequential prediction tasks in IoT. 
% FedSTL first identifies client properties as Signal Temporal Logic (STL) specifications. Then, a partitioning component of FedSTL associates each client to an aggregation center, while the framework continues to infer properties for the cluster. 
% At the training stage, both cluster and client models are encouraged to follow customized properties to achieve a hierarchical property enhancing strategy. 
% Further, we show preliminary results of FedSTL in this work for two sequential prediction datasets. 

Recent advancements in federated learning (FL) have greatly facilitated the development of decentralized collaborative applications, particularly in the domain of Artificial Intelligence of Things (AIoT). However, a critical aspect missing from the current research landscape is the ability to enable data-driven client models with symbolic reasoning capabilities. Specifically, the inherent heterogeneity of participating client devices poses a significant challenge, as each client exhibits unique logic reasoning properties. Failing to consider these device-specific specifications can result in critical properties being missed in the client predictions, leading to suboptimal performance. In this work, we propose a new training paradigm that leverages temporal logic reasoning to address this issue. Our approach involves enhancing the training process by incorporating mechanically generated logic expressions for each FL client. Additionally, we introduce the concept of aggregation clusters and develop a partitioning algorithm to effectively group clients based on the alignment of their temporal reasoning properties. 
We evaluate the proposed method on two tasks: a real-world traffic volume prediction task consisting of sensory data from fifteen states and a smart city multi-task prediction utilizing synthetic data. The evaluation results exhibit clear improvements, with performance accuracy improved by up to 54\% across all sequential prediction models.

% We propose FedSTL, a personalized federated learning (PFL) framework that enhances client-level customized properties using mined temporal specifications with a hierarchical clustering structure. 
 % Existing developments on PFL achieve model personalization by using auxiliary global models to guide the learning of client models. 
% Existing developments on PFL train personalized client models under the guidance of auxiliary global models. 
% However, it is challenging for the previous approaches to encourage client models to follow their individualized data properties while maintaining model personalization. 
% Moreover, specifying client-level property is very challenging for larger-scale PFL applications where the properties vary drastically for each participating client. 
% To this end, the proposed FedSTL framework first mines individualized model properties for each PFL client based on private data points. 
% Then, a iterative clustering algorithm is designed with a hierarchical client property enhancement mechanism to enhance both the cluster property and the client property during training. 
% Empirical evaluations show that FedSTL is able to improve the client-level model property satisfaction while boosting the prediction accuracy compared with existing frameworks. 
\end{abstract}

\section{Introduction}

In recent years, modern artificial intelligence (AI) models have been adapted to handle massively-distributed tasks deployed on non-independent and identically distributed (i.i.d) devices, including AI-empowered Internet-of-Things (IoT) services such as smart cities and smart healthcare~\cite{nguyen2021federated,ma2019data,ma2021toward, preum2021review}. However, a critical challenge that remains unsolved is the ability to enable large-scale distributed data-driven models with \textit{symbolic reasoning} capabilities.

More specifically, federated learning (FL)~\cite{mcmahan2017communication} frameworks have been developed to achieve remarkable performance for distributed training with reduced privacy risks. These frameworks enable collaborations between participating devices by aggregating their models through a centralized moderator. FL models have proven particularly useful in privacy-sensitive outdoor sensor networks, where participating clients can receive non-i.i.d or heterogeneous input data. 
Client data is securely maintained on-device, ensuring that sensitive information remains local and private. To facilitate model updates and collaboration, only client model parameters are synchronized with the centralized moderator.

\begin{figure}
\centering
    \includegraphics[width=.48\textwidth]{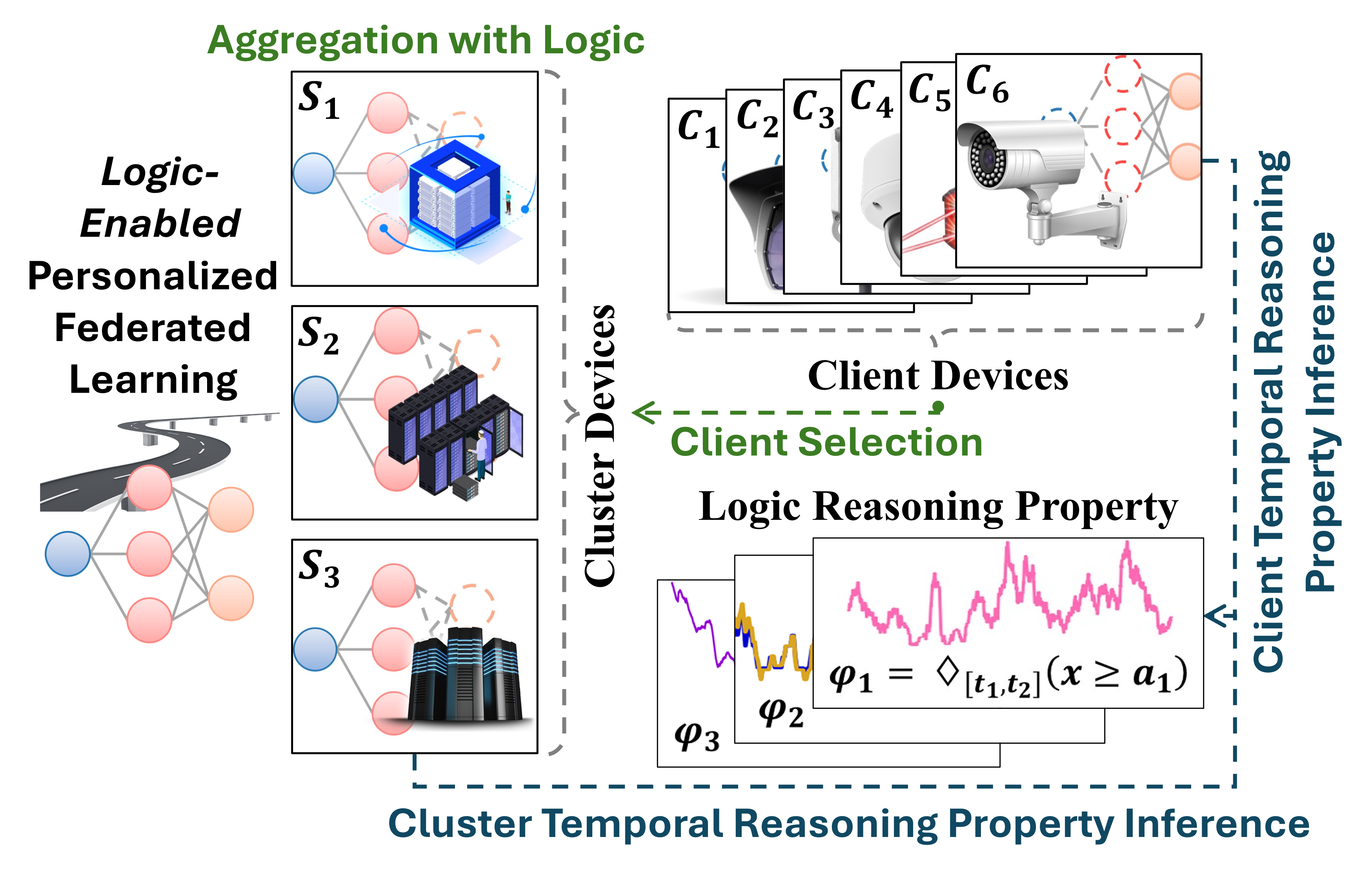}
    \caption{FedSTL consists of $\mathcal{S}$ cluster devices, and $\mathcal{C}$ client devices. During each communication round, client models are partitioned into clusters. Then, inferred temporal reasoning properties are enhanced on predictive models.}
    \label{fig:overview}
    \vspace{-1em}
\end{figure}

However, dealing with client heterogeneities symbolically remains a challenge for FL models. Previous FL designs have attempted to address this issue through methods such as client selection, clustering, and regularization~\cite{collins2021exploiting, li2021ditto, li2021fedbn, mohri2019agnostic, yu2020salvaging}. Unfortunately, these methods fail to provide any reasoning for the model's outputs. Furthermore, no logic reasoning properties or domain-specific knowledge can be integrated and enforced through such approaches.

Consider a sequential real-world prediction task that involves multiple types of sensors (e.g., radar sensors, video detection systems, air quality monitors) deployed at various locations throughout a city (e.g., highways, town roads, indoors)~\cite{hpmsfield,ma2021novel,ma2018cityresolver,ma2017cityguard}. The heterogeneity of deployment locations and sensor types leads to monitoring data with diverse distributions, with each client following their unique time-series patterns. Previous FL approaches have failed to provide any symbolic reasoning capabilities that enable client predictions to follow these distinct time-series patterns.

Additionally, current FL frameworks lack the ability of \textit{intra-task symbolic reasoning}. In the case of the previous example, multi-class prediction tasks, such as road occupancy and noise level, can exhibit correlations, but this correlation may not be consistent across all clients. A more advanced FL framework should incorporate the ability to understand how diverging intra-task logical reasoning patterns as such can be aligned with client predictions.

Motivated by these challenges, we present a new FL framework towards distributed temporal symbolic reasoning dubbed \textit{FedSTL}, short for \textit{signal temporal logic-enabled personalized federated learning}. Figure~\ref{fig:overview} depicts the overall structure of our approach. In FedSTL, the task is to predict future multivariate traces based on local datasets for each client. We aim to tackle two tasks in FL: 1) identifying and regulating client models with localized temporal reasoning properties at the training time, and 2) aggregating each client with others who have similar properties. 

% \begin{wrapfigure}{l}{0.6\textwidth}
%     \centering
%     \includegraphics[width=.6\textwidth]{figs/demo.png}
%     \caption{Each subplot shows the traffic pattern of a monitoring station from Dec 09, 2019 to Dec 13, 2019 in Washington State. The first 96 hours are used as the input sequence, and the later 24-hour traffic volume is the model prediction. }
%     \label{fig:datap}
%     % \vspace{-0.5cm}
% \end{wrapfigure}

We use the term ``property" to refer to any prior domain knowledge, world knowledge, or expertise knowledge. However, manually specifying comprehensive multivariate properties for large-scale FL systems is another significant challenge, where each client differs from the others. Hence, our model enables extracting localized client knowledge \textit{automatically} at training time to address this issue. 
By properties and specifications, we refer to patterns derived from the private data of each client. In Table~\ref{tab:stl-case}, we provide a more comprehensive set of examples demonstrating the types of properties and specifications that can seamlessly integrate into our framework.

In the personalized FL framework, each client has a locally distinct model responsible for making predictions based on their private input data~\cite{chen2022pfl}. The central aggregator(s) then leverage updates from multiple client models to improve the central model. However, in our framework, we introduce two modifications to this popular learning paradigm. Firstly, we incorporate client prediction regulations with automatically inferred logic reasoning. Secondly, we use temporal logic properties to cluster clients based on the alignment of these properties. Within each cluster, the members are aggregated to contribute to the same shared model. Additionally, our framework ensures \textit{multi-granularity personalization} by enforcing specialized cluster and client properties. 

By regulating client predictions with locally inferred properties, we enhance these properties for each client using a teacher-student learning paradigm, aligning the prediction results more closely with the specified requirements. This enables us to effectively address the challenge posed by heterogeneity in client deployment locations and sensor types, which leads to different data distributions for each client. In addition, we employ clustering based on property alignments, allowing us to aggregate clients with similar properties together. With this approach, we group clients whose properties exhibit higher similarity, thus reducing the aggregation of weights from non-agreeing clients.
Throughout this work, we use the terms ``property" and ``specification" interchangeably. We summarize the \textit{contributions} as follows:
\begin{itemize}[wide=0pt, leftmargin=*, topsep=0pt]
    \item FedSTL is a novel personalized FL framework that enhances temporal reasoning through automatically inferred logic properties for non-agreeing FL clients. 
    \item Our framework is designed to facilitate the automatic discovery and induction of client and cluster temporal logic specifications from datasets.
    \item FedSTL clusters FL clients based on agreements of specifications, enabling cross-client collaboration for cluster models to exploit shared knowledge.
    \item We evaluate FedSTL under various dataset settings, including two realistic testing scenarios with both real-world data and simulated datasets. Empirical evaluations demonstrate that FedSTL improves client-level model property satisfaction while boosting prediction accuracy compared to existing frameworks.
\end{itemize}

% The rest of this paper is structured as follows: 
% In Section~\ref{sec:problem}, we formally define the problem setting of FedSTL. 
% In Section~\ref{sec:theo}, we introduce the theoretical background and formal definitions used in our framework. 
% In Section~\ref{sec:method}, we provide a detailed description of the implementation. 
% We present evaluation results in Section~\ref{sec:eval}. 
% Then, in Section~\ref{sec:related}, we discuss related work on FL and prior knowledge-guided learning paradigms.
% In Section~\ref{sec:broaderimp}, we discuss the broader impact of our proposed framework. Finally, we conclude our study in Section~\ref{sec:conclusion}.

\paragraph{Notations}
$\{ \mathcal{C}_i \}$: client models; $\{ \mathcal{S}_j \}$: cluster models; $\{\mathcal{X} \times \mathcal{Y} \} \in \mathcal{D}_i$: client datasets; $\theta$: model parameters; $F$: local objective function; $G$: aggregation function; $\square$ \textit{always}; $\Diamond$: \textit{eventually}; $\mathcal{U}$: \textit{until}; $\mu$: predicate variable; $\varphi$: STL formula; $\varphi(\alpha)$: templated STL formula; $\rho$: STL robustness.

\begin{table*}[t]

\small 

\centering

\begin{tabular}{p{9cm}|p{5.6cm}|p{1.5cm}}

\toprule

\textbf{Temporal Reasoning Property}  & \textbf{Templated Logic Formula} & \textbf{Parameters} \\

\midrule

\textbf{1: Operational Range}: Signal is upper-bounded by threshold $a$ and lower-bounded by threshold $b$. 
&
$\bigwedge_{i=1}^{1, 2, ..., \tau} (\square_{[i,i+t]} (x\leq a_i\wedge x\geq b_i) ) $
& $a_i, b_i $ \\

\midrule 

\textbf{2: Existence}: Signal should eventually reach the upper extreme $a$ and the lower extreme $b$.
&
$\bigwedge_{i=1}^{1, 2, ..., \tau} (\Diamond_{[i,i+t]} (x \leq a_i \wedge x \geq b_i) ) $ 
& $a_i, b_i$ \\

\midrule 

\textbf{3: Until}: Signal must satisfy one specification at all times until another condition is met.
&
$\bigwedge_{i=1}^{1, 2, ..., t} ((x < a_i) \, \mathcal{U}_{[i,i+1]} (x < b_i))$ 
& $a_i, b_i$\\

\midrule

\textbf{4: Intra-task Reasoning}: The difference between signal variables $x_1$ and $x_2$ should be greater than $a$.
&
$\bigwedge_{i=1}^{1, 2, ..., \tau} (\square_{[i,i+t]}((x_1-x_2)>a_i))$ 
& $a_i$ \\

\midrule

\textbf{5: Temporal Implications}: The happening of one event indicates that another event will happen at some point in the future.
&
$\square_{[t_1,t_2]} ((x \geq a_1) \rightarrow \Diamond_{[t_3,t_4]}(x \geq a_2))$ 
& $a_1, a_2$ \\

\midrule

\textbf{6: Intra-task Nested Reasoning}: The signal variable $x_1$, when greater than a threshold $a$, indicates $x_2$ will eventually reach a threshold $b$. 
&
$\square_{[t_1,t_2]}((x_1 \geq a) \rightarrow \Diamond_{[t_3,t_4]} (x_2 \geq b))$ 
& $a, b$ \\

\midrule

\textbf{7: Multiple Eventualities}: Multiple events must eventually happen, but their order can be arbitrary.
&
$\Diamond_{[t_1,t_2]} (x \geq a_1) \wedge \cdots \wedge \Diamond_{[t_3,t_4]} (x \geq a_n)$ 
& $a_1, a_2 \cdots a_n$ \\

\midrule

\textbf{8: Template-free}: specification mining without a templated formula.
&
No pre-defined templates are needed.
&  n/a. \\

\bottomrule

\end{tabular}

\caption{Examples of temporal reasoning templates specified with STL.}
\label{tab:stl-case}
\end{table*}

\section{Problem Formulation}\label{sec:problem}

A general FL framework consists of a central server and $\{ \mathcal{C}_i \}$ client devices, each with its own local dataset $\mathcal{D}_i$ consisting of i.i.d data points $(x,y)$. The primary objective is to train a central model parameterized by $\theta_g$ without requiring the clients to share their private data.

During each communication round, the central server first broadcasts the current version of $\theta_g$ to the participating clients. The clients then use local datasets $\mathcal{D}_i$ to train the model for a certain number of iterations. Then, the clients send the updated model parameters $\{ \theta_i\}, i\in \mathcal{C}$ back to the central server, which performs an aggregation of these updates to generate a new version of $\theta_g$. This process is repeated for multiple times.
% until some criteria, such as a certain level of accuracy or a maximum number of rounds, are met. 

In this work, our focus is on training personalized parametric models $\{\theta_i \}$ for clients using a slightly different setup. More concretely, we partition client models into $\{ \mathcal{S}_j \}$ clusters, each with its own model parameter $\theta_j$. 
During each communication round, the cluster models $\{ \theta_j \}$ are broadcast to the clients assigned to the corresponding cluster. Clusters and clients then follow a similar back-and-forth communication rule as other general FL training paradigms. 

% Equation~\ref{eq:genobj} defines the training objective of our personalized setting PFL. 
% \begin{equation}\label{eq:genobj}
% \begin{gathered}
    % \mathop{\text{min}}_{{v_1 \ldots v_N}} \, h_k(v_k;w^*) \coloneqq F_k(v_k) + \lambda R(v_k;w^*) \\ 
    % s.t. \qquad w^* \in \text{arg} \mathop{\text{min}}_{{w}} \, G(F_1 (w) ,\ldots, F_N (w))
% \end{gathered}
% \end{equation}

% Let $f(\theta; x, y) \rightarrow \mathbb{R}$ be the loss of model $\theta$ at the data point $(x,y)$. Let $F^i (\theta) \coloneqq \mathbb{E}_{(x,y) \sim \mathcal{D}_i} [f(\theta; x, y)] \allowbreak \forall i \in \mathcal{C}$, which is
% the local objective for client $i$. 
% Formally, the aim is to find better client models $\{ \hat{\theta}_i \}$ that are close to the optimal models $ \theta_i^* \in \text{argmin}_{\theta} F^i (\theta), \, i \in \mathcal{C}$. We use $G$ to denote an aggregation function, which defines how client weights are updated to the global. 

Let $f(\theta; x, y) \rightarrow \mathbb{R}$ denote the loss of model $\theta$ at the data point $(x,y)$. For each client $i \in \mathcal{C}$, let $F_i (\theta_i) \coloneqq \mathbb{E}_{(x,y) \sim \mathcal{D}_i} [f(\theta_i; x, y)]$ be the local objective function.
Our goal is to obtain better client models $\{ \hat{\theta}_i \}$ that are close to the optimal models $\theta_i^* \in \text{argmin}_{\theta_i} F_i (\theta_i)$ for each $i \in \mathcal{C}$. Additionally, we use $G(\cdot)$ to denote an aggregation function that defines how the client model updates are combined to form the global model update.

\section{Temporal Reasoning Property Inference}\label{sec:theo}

\subsection{Temporal Logic Specification}

We first introduce the preliminaries of signal temporal logic (STL)~\cite{maler2004monitoring}, which is a formalism that provides a flexible and rigorous way to specify temporal logic reasoning. To begin, we provide the syntax of an STL formula, as defined in Definition~\ref{def:stl}. 

\begin{definition}[STL syntax]\label{def:stl}
\[
\varphi 
::= 
\mu 
\mid \neg \mu
\mid \varphi_1\wedge \varphi_2
\mid \varphi_1 \vee \varphi_2  \]
\[ \mid \Diamond_{[a,b]}\varphi 
\mid \square_{[a,b]}\varphi
\mid \varphi_1 \mathcal{U}_{[a,b]} \varphi_2 \]

\end{definition}

We use the notation $[a, b] \in \mathbb{R}_{\geq 0}$, with $a \leq b$, to represent a temporal range. Let $\mu: \mathbb{R}^n \rightarrow \{ \top, \bot \}$ be a signal predicate (e.g. $f(x) \geq 0$) on the signal variable $x \in \mathcal{X}$. Additionally, we refer to different STL formulas using $\varphi$, $\varphi_1$, and $\varphi_2$. We use $\square$ to denote the property ``always," which requires the formula $\varphi$ to be true at all \textit{future} time steps within $[a, b]$. Similarly, we use $\mathbin{\Diamond}$ to denote ``eventually," which requires the formula $\varphi$ to be true at some future time steps between $[a, b]$. Finally, we use $\mathcal{U}$ to denote ``until," which specifies that $\varphi_1$ is true until $\varphi_2$ becomes true.

An example of an STL formula is $\square_{[0, 5]}( (x_1 \geq 0.75) \rightarrow (x_2 \geq 10 ))$, which formally specifies that if the signal variable $x_1$ exceeds or equals to $0.75$ during future times $[0, 5]$, then the signal variable $x_2$ should always be greater than or equal to $10$. 

\subsection{Logic Inference Through Observed Data}\label{sec:method-mining}

\paragraph{Logic inference} Logic inference~\cite{bartocci2022survey} is the process of generating logic properties based on observed facts when the desired system property is unknown or only partially available. Given some prior knowledge about the possible form of a logic property, the logic inference algorithm (specification mining~\cite{jha2017telex}) learns the complete logic formula. 
Formally, the logic property inference task is described in Definition~\ref{def:stlinf}. 
\begin{definition}[STL property inference]\label{def:stlinf}
Given an observed fact $x$ and a templated STL formula $\varphi(\alpha)$, where $\alpha$ is an unknown parameter, the task is to find a value for $\alpha$ such that $\varphi$ is satisfied for all instances of $x$. 
\end{definition}

We provide a practical example of STL property inference in Example~\ref{eg:tempstl}. 
\begin{example}[An example of STL inference]\label{eg:tempstl}
    \textnormal{Given a templated STL property} $\varphi_k (\alpha) = \square_{[0, 5)}( (x_1 \geq 0.75) \rightarrow  (x_2 \geq \alpha ))$, \textnormal{the goal is to find a value for the unknown parameter $\alpha$ such that during future time stamps $0$ to $5$, if the signal variable $x_1$ exceeds or equals to $0.75$, the signal variable $x_2$ should always be greater than or equal to $\alpha$. }
\end{example}

Furthermore, we have summarized eight categories of temporal reasoning properties in Table~\ref{tab:stl-case} that can be expressed using STL and can be inferred through specification mining algorithms.

In practice, there are infinite possible values that a free parameter (e.g., $\alpha$) can take to make a templated formula $\varphi(\alpha)$ valid. However, not all valid values of $\alpha$ are equal in terms of enhancing the reasoning property during the training process. To be more specific, we need to find an $\alpha$ value such that the observed facts satisfy $\varphi(\alpha)$ with a small margin. 
In the property inference process, the STL quantitative semantics (robustness) as defined in~\cite{donze2010robust} is used as a real-valued measurement for property satisfaction. 

% \begin{definition}[STL Robustness $\rho$]\label{def:stlrob}
% Formally, the degree of satisfaction of an STL formula with respect to a signal $\mathbf{x}$ at time $t$ is defined inductively as follows:
% \begin{align*}
% & ~p \vDash \varphi & \Leftrightarrow & \quad\left(p, 0\right) \vDash \varphi\\
% & \left(p, t\right) \vDash \varphi_1 \wedge \varphi_2 \quad & \Leftrightarrow & \quad\left(p, t\right) \vDash \varphi_1 \wedge\left(p, t\right) \vDash \varphi_2\\
% & \left(p, t\right) \vDash \varphi_1 \vee \varphi_2 \quad & \Leftrightarrow & \quad\left(p, t\right) \vDash \varphi_1 \vee \left(p, t\right) \vDash \varphi_2\\
% & \left(p, t\right) \vDash \Diamond_{[a, b]} \varphi \quad & \Leftrightarrow & \quad \exists t^{\prime} \in\left[t+a, t+b\right],\left(p, t^{\prime}\right) \vDash \varphi\\
% & \left(p, t\right) \vDash \square_{[a, b]} \varphi \quad & \Leftrightarrow & \quad \forall t^{\prime} \in\left[t+a, t+b\right], \left(p, t^{\prime}\right) \vDash \varphi\\
% & \left(p, t\right) \vDash \varphi_1 \mathcal{U}_{[a, b]} \varphi_2 & \Leftrightarrow & \quad\exists t^{\prime} \in\left[t+a, t+b\right], \left(p, t^{\prime}\right) \vDash \varphi_2
% \\&&& \quad
% \wedge \forall t^{\prime \prime} \in\left[t, t^{\prime}\right],\left(p, t^{\prime \prime}\right) \vDash \varphi_1
% \end{align*}
% \end{definition}

In the following example, we briefly show how to utilize the robustness value as a real-valued measurement for property satisfaction. In this instance, we provide a 5-step sequential data, which is then evaluated against the temporal property listed in Example 1.
Continuing with the templated STL defined in Example~\ref{eg:tempstl}, when $\alpha=10$, the 5-step sequential data $\mathcal{X} = ((0.25, 20),\, (0.25, 18),\, (0.5, 16),\, (0.6, 14),\, (0.75, 12))$ results in a robustness value of $\rho (\varphi, \mathcal{X}) = -2$. 
The resulting robustness value is -2, indicating that the property was not satisfied.
However, a value close to $\alpha=12$ would fit $\mathcal{X}$ tightly, resulting in an STL property that more accurately describes the true observation. Therefore, the STL inference task can be better formulated as Definition~\ref{def:stltight}, which incorporates the notion of a tight bound.

\begin{definition}[STL property inference with a tight bound]\label{def:stltight}
Given an observed data $\mathcal{X}$ and a templated STL formula $\varphi (\alpha)$, where $\alpha$ is an unknown parameter. The goal is to identify a value for $\alpha$ that results in a tightly-fitted logic property, expressed by the equation $\rho (\varphi, \mathcal{X}; \alpha) = \epsilon$, where a smaller positive $\epsilon$ indicates a closer alignment with the data. 
\end{definition}

The task defined in Definition~\ref{def:stltight} can be effectively solved using the following algorithm, which can be solved with either gradient-free or gradient-based numerical optimization methods~\cite{jha2017telex}. 
\begin{equation}\label{eq:solvestl}
\begin{gathered}
    \mathop{\text{min}}{|\epsilon|} \; s.t. \; \epsilon=p^\prime-p  \\ 
    \text{where} \; \rho(\varphi(p),\mathcal{X},t) \geq 0 \; \text{and} \; \rho(\varphi(p^\prime),\mathcal{X},t) < 0
\end{gathered}
\end{equation}

In the templated logic formula $\varphi$, let $p$ and $p^\prime$ denote candidate values of free parameters, and let $t$ be a timestamp.  our objective is to minimize the value of $|\epsilon|$, which represents the discrepancy between a satisfactory parameter value and an unsatisfactory parameter value. 
The STL robustness function is not differentiable at zero. Therefore, to tackle this issue, alternatives such as the ``tightness metric"~\cite{jha2017telex} can be employed to effectively address this problem.

% In practice, it is unrealistic to assume that every client should follow the same set of properties at all times. 
% For example, Figure~\ref{fig:datap} shows that the traffic volume at different sensing stations follows distinct patterns and has various magnitudes.
% In our running example, this can be affected by time-sensitive factors such as the weather, holidays, and times of the day. 
% Moreover, for a multi-task traffic prediction scenario, the correlation between road occupancy and mean speed varies according to factors such as speed limit and road function types (e.g., interstate, arterial, and local roads). 
% Therefore, it is very challenging to define data properties for every IoT device in the network that correctly reflect temporal distinctions. 
% To this end, FedSTL formally constructs data properties from templated STL formulas, following the idea of property mining in~\cite{jha2017telex}. 
\section{Logic-Enabled Federated Learning}\label{sec:method}

\begin{figure*}[t]
    \centering
    \includegraphics[width=\textwidth]{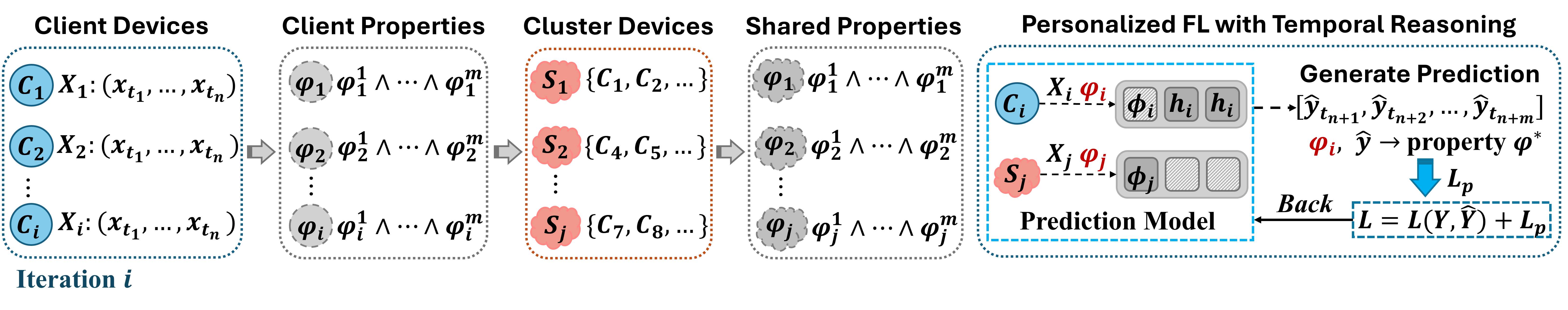}
    \caption{The training workflow for one iteration in the framework involves the inference of client logic properties $\{ \varphi_i\}$ and cluster logic properties $\{ \varphi_j\}$. Based on the alignment of these properties, clients are partitioned into clusters. Then, our framework enhances personalized FL by incorporating both client and cluster reasoning properties during training.}
    \label{fig:learning}
    \vspace{-0.5em}
\end{figure*}

\subsection{Enhancing STL-based Logic Reasoning Property}\label{sec:method-enhancing}
%  At the beginning of each training iteration, client devices draw training data from their private datasets. Then, a temporal logic inference algorithm is applied to extract logic reasoning properties from the observed data points. These generated logic reasoning properties are subsequently used to regulate client models within a teacher-student structure.

We consider the workflow illustrated in Figure~\ref{fig:learning}. During each training iteration, FedSTL utilizes Equation~\ref{eq:solvestl} to infer a logic reasoning property $\varphi$ from the client dataset $D_i$. The inferred property is then incorporated into the client prediction process through an additional loss term $\mathcal{L}_p$, which penalizes the neural network for any deviations from the property. This is achieved by implementing a teacher-student structure~\cite{hinton2015distilling}, which will be discussed in detail in subsequent sections. 

\begin{algorithm}[t]
\small
    \caption{\small CLUSTER\_ID: Cluster identity mapping}
    %\meiyi{Describe it with natural language first}
    \label{alg:cluster}
    \begin{flushleft}
    \textbf{Parameters}: cluster devices $\{\mathcal{S}_j\}$, participating client models $\{\mathcal{C}_i\}$.\\
    \begin{algorithmic}[1] %[1] enables line numbers
        \STATE Initialize client identity mapping $\mathcal{I}$.
        \\ \texttt{/* generate client data property */} 
        \FOR{client dataset $\mathcal{D}^s_i$}
            \STATE Generate client data property $\varphi^s_i$ on dataset $\mathcal{D}^s_i$
        \ENDFOR
        \\ \texttt{/* clusters select clients */} 
        \FOR{cluster device $\mathcal{S}_j$}
            \FOR{client data $\mathcal{D}^s_i$}
            \STATE Generate $\hat{\mathcal{Y}}$ with $\mathcal{S}_j$ and $\mathcal{D}^s_i$
            \STATE Calculate empirical logic reasoning loss $\mathcal{L}_p(\varphi^s_i, \hat{\mathcal{Y}})$
            \ENDFOR
            \STATE Cluster selects clients $\mathcal{C}_i$ with the lowest $\mathcal{L}_p$
            \STATE Append $\mathcal{C}_i$ to identity mapping $\mathcal{I}$. 
        \ENDFOR 
        \STATE \textbf{return} $\mathcal{I}$ (client identity mapping)
    \end{algorithmic}
    \end{flushleft}
\end{algorithm}

\paragraph{Client model prediction correction} 
We first describe how FedSTL framework regulates and corrects client prediction such that it satisfies the logic property $\varphi \coloneqq \varphi_1 \wedge \varphi_2 \wedge \ldots \wedge \varphi_n$ extracted by the logic inference module.

To begin with, recall that any locally inferred logic reasoning property must be satisfied by every data point in the client dataset. Additionally, any signal temporal logic formula can be represented by its equivalent Disjunctive Normal Form (DNF), which typically takes the form of $P \vee Q \vee R \vee \ldots$. Each clause in a DNF formula consists of either variables, literals, or conjunctions, for example, $P \coloneqq p_1 \wedge \neg p_2 \wedge \ldots \wedge p_{n}$. Essentially, the DNF form specifies a range of satisfaction where \textit{any} logic clause connected with the disjunction operator satisfies the STL formula $\varphi$.

Taking advantage of this fact, our framework leverages the DNF equivalents of automatically generated client properties. We aim to find a logic clause $\varphi^*$ that represents the closest approximation to the model prediction $\hat{\mathcal{Y}}$ in terms of satisfying the property~\cite{ma2020stlnet}. To achieve this, we introduce an additional loss term $\mathcal{L}_p(\varphi^*, \hat{\mathcal{Y}})$ that quantifies the distance between the model prediction and the closest satisfying trace. In practice, we use the L-1 distance as a metric to quantify this loss, capturing the absolute difference between the predicted and desired outputs.

\subsection{Dynamic Temporal Logic-Based Clustering}\label{sec:method-clustering}

% Our proposed framework then generates client-cluster identities based on the satisfaction of client logic properties. Once the client-cluster identity is inferred, clusters follow the property inference algorithm described earlier and generate a set of corresponding cluster properties. Finally, both client properties and cluster properties are used to guide a bi-level training paradigm, as described in later sections.

One advantage of the FedSTL framework is its dynamic assignment of client models to clusters, which allows the aggregation process to adapt to changes in logic properties as they occur. At a high level, clients with similar temporal reasoning properties are grouped together in clusters, while clients with different properties are not aggregated together.

We demonstrate this process in Algorithm\ref{alg:cluster}, which is executed every $m$ communication rounds, where $m$ is a hyperparameter. Let ${\mathcal{D}^s_i}, {i \in \mathcal{C}}$ be a small sample of desensitized client data. During each clustering process, FedSTL generates the logic property $\varphi^s_i$ for each participating client on their respective datasets $\mathcal{D}^s_i$ (line 3, Alg.\ref{alg:cluster}). Then, each cluster device $\mathcal{S}_j$ generates predictions $\hat{\mathcal{Y}}$ on the samples in $\mathcal{D}^s_i$ (line 7, Alg.\ref{alg:cluster}). Next, we calculate the empirical logic reasoning loss $\mathcal{L}_p(\varphi^s_i, \hat{\mathcal{Y}})$ for each cluster model $\mathcal{S}_j$ on client dataset $\mathcal{D}^s_i$ (line 8, Alg.\ref{alg:cluster}). At the end of each round, we select cluster members based on the lowest logic reasoning loss $\mathcal{L}_p$ (line 10, Alg.~\ref{alg:cluster}).

% FedSTL performs client model clustering based on desensitized datasets $\{D^s_k \}, {k \in K}$ and cluster models $C$ (line 3, Algorithm~\ref{alg:fedstl}).  
% In the selection process, FedSTL first generates data properties $\varphi_i$ for each desensitized data $D^s_i$ of the participating clients (line 3, Algorithm~\ref{alg:cluster}). 
% Then, each cluster model $M_j$ generates prediction $\hat{Y}_{i,j}$ on dataset $D^s_i$ (line 7, Algorithm~\ref{alg:cluster}). 
% Finally, the empirical property loss $\mathcal{L}_P(\varphi_i, \hat{Y}_{i,j})$ $w.r.t$ cluster model $M_j$ and client dataset $D^s_i$ is estimated (line 8, Algorithm~\ref{alg:cluster}). Cluster members are selected based on the lowest empirical $\mathcal{L}_P$ at each round (line 10, Algorithm~\ref{alg:cluster}). 

\begin{algorithm}[t]
\small
    \caption{\small FedSTL: Client federation and update}
    \label{alg:fedstl}
    \begin{flushleft}
    \begin{algorithmic}[1] %[1] enables line numbers
        \FOR{$t=1,2,\ldots, \mathcal{T}$} 
        \STATE $\{\mathcal{C}^{(t)}_i\} \leftarrow$ Selected clients with participation rate $r$
        \STATE $\mathcal{I}^{(t)} \leftarrow $  {\scshape CLUSTER\_ID}$(\{\mathcal{S}_j\}, \{\mathcal{C}^{(t)}_i\})$ 
        \STATE Clusters broadcast current model $\phi^{(t)}_j$ to clients
        \\ \texttt{/* update client models */} 
        \FOR{client $i$ in $\{\mathcal{C}^{(t)}_i\}$ \textit{in parallel} }
        \STATE Client $i$ initializes layers $\theta^{(t)}_i$
            \FOR{$t=1,2,\ldots,\tau$}
                \STATE $(\theta^{(t)}_i) \leftarrow \text{{\scshape SGD}}(\theta^{(t)}_i, F_i, \eta)$
            \ENDFOR
        \STATE Client $i$ sends shared $\phi^{(t)}_i$ to cluster $\mathcal{S}_{j}$
        \ENDFOR
        \\ \texttt{/* update cluster models */} 
        \FOR{cluster $j$ in $\{\mathcal{S}^{(t)}_j\}$ \textit{in parallel}}
            \STATE Cluster $j$ performs member aggregation:
            \\ \hskip1.0em $\phi^{(t)}_j \leftarrow$ {\scshape G}$(\phi^{(t)}_j, \{ \phi^{(t)}_{i} \})$
            \FOR{$t=1,2,\ldots,\kappa$}
                \STATE $(\phi^{(t+1)}_j) \leftarrow \text{{\scshape SGD}}(\phi^{(t)}_j, F_j, \eta_j)$
            \ENDFOR
        \ENDFOR
        \ENDFOR
        \STATE \textbf{return} $ { \{ \phi_j\}, \{ \theta_i \} }$ (updated models)
    \end{algorithmic}
    \end{flushleft}
\end{algorithm}

\subsection{Hierarchical Logic Reasoning Strengthening}\label{sec:method-training}
The pseudocode for the collaborative updating paradigm of FedSTL is presented in Algorithm~\ref{alg:fedstl}. The framework operates for a total of $\mathcal{T}$ communication rounds, and at each round, participating client models $\{\mathcal{C}^{(t)}_i\}$ are selected based on a pre-defined participation rate $r$ (line 2, Alg.~\ref{alg:fedstl}). Our approach employs a bi-level updating strategy, where client and cluster models are updated differently.

The client model parameters $\{\theta_i\}, i \in \mathcal{C}$ are divided into two groups: cluster-shared parameters $\{\phi_i\}$ and locally-private parameters $\{h_i\}$. The former are updated and transferred to the cluster devices, while the latter are retained on the client devices for personalization. For instance, in a recurrent neural network used for sequential prediction, the recurrent blocks can be designated as $\{\phi_i\}$ to capture the shared characteristics among cluster members. Meanwhile, the local client parameters $\{h_i\}$ can be a dense layer that maps the output of the recurrent layer to the prediction.

During each communication round, client models obtain the most recent client model $\theta_i^{(t)}$ by minimizing the local objective defined in Equation~\ref{eq:localobj} (line 8, Alg.~\ref{alg:fedstl}), where $\tau$ denotes the number of local updates for client models. Specifically, the local objective $F_i$ is defined as:
\begin{equation}\label{eq:localobj}
\mathop{\text{min}}_{{\theta_i}} \, F_i(\theta_i) \; \text{with} \; F_i(\theta_i)\coloneqq \mathcal{L} (\mathcal{Y}, \hat{\mathcal{Y}}) + \lambda \, \mathcal{L}_{p} (\varphi_i, \hat{\mathcal{Y}})
\end{equation}
Here, $\mathcal{L}$ is a local loss function, such as Mean Squared Error, and $\mathcal{L}_{p}$ is an additional loss function with respect to the client STL property $\varphi_i$. The hyperparameter $\lambda$ is used to control the strength of the property loss.

% Then, customized personal layers that are kept privately can be further trained for $\zeta$ rounds with Equation~\ref{eq:localobj} . 
We employ stochastic gradient descent ({\scshape SGD}) as the optimization algorithm to update the neural network models (line 8, Alg.~\ref{alg:fedstl}), as shown in Equation~\ref{eq:sgd}, where $\eta$ denotes the step size for the gradient descent. The {\scshape SGD} update rule can be substituted with any other gradient descent-based algorithm. After local client updating, the shared layers $\{ \phi_i \}$ are uploaded to the cluster model (line 10, Alg~\ref{alg:fedstl}). 
\begin{equation}\label{eq:sgd}
(\theta^{(t)}_i) \leftarrow \text{{\scshape SGD}}(\theta^{(t)}_i, F_i, \eta)
\end{equation}

During the cluster updating rounds, each cluster model aggregates the updated layers from its members using the function {\scshape $G(\cdot)$} (line 13, Alg.~\ref{alg:fedstl}). In FedSTL, clusters compute $G(\cdot)$ directly as a weighted average. Finally, in order to further exploit the shared logic reasoning property among cluster devices, the framework performs $\kappa$ rounds of updates on the cluster models while enforcing the inducted STL constraint $\{ \varphi_j \}$, where $\varphi_j$ is a temporal reasoning property inducted for the $j^{th}$ cluster. This is done by optimizing the objective specified in Equation~\ref{eq:clustobj}, and the process is described in line 15 of Alg.~\ref{alg:fedstl}.
\begin{equation}\label{eq:clustobj}
\begin{gathered}
 \mathop{\text{min}}_{{\phi_j}} \, F_j(\phi_j), \; \text{with} \; F_j(\phi_j)\coloneqq \mathcal{L} (\mathcal{Y}, \hat{\mathcal{Y}}) + \lambda \, \mathcal{L}_{p} (\varphi_j, \hat{\mathcal{Y}}) \\ 
(\phi^{(t)}_j) \leftarrow \text{{\scshape SGD}}(\phi^{(t)}_j, F_j, \eta_j)
\end{gathered}
\end{equation}

Importantly, synchronizing client parameters with clusters does not require client personalization on $\{ h_{i} \}$ to be completed first, as the personalized layers are not shared among clients. This means that client models can continue to perform local personalization, even if they are selected to participate in a given communication round.

\begin{table*}[t]
\centering 
\small

\begin{tabular}{l|cc|cc|cc|cc}

\toprule

\multirow{3}{*}{\textbf{Method}}& \multicolumn{2}{c|}{\multirow{2}{*}{\textbf{RNN}}} & \multicolumn{2}{c|}{\multirow{2}{*}{\textbf{GRU}}} & \multicolumn{2}{c|}{\multirow{2}{*}{\textbf{LSTM}}} & \multicolumn{2}{c}{\multirow{2}{*}{\textbf{Transformer}}} \\
 % \cline{3-3} 
                                 & \multicolumn{2}{l|}{}                              & \multicolumn{2}{l|}{}                              & \multicolumn{2}{l|}{}   \\ 

                                 & MSE   & $\rho \%$ & MSE & $\rho \%$ & MSE & $\rho \%$ & MSE & $\rho \%$ \\ 
\midrule
\multirow{1}{*}{\textbf{FedAvg}}    & .128$\pm$.032 & 78.96$\pm$1.03 & .154$\pm$.031 & 80.51$\pm$0.75 & .126$\pm$.034 & 81.80$\pm$0.91 & .588$\pm$.005 & 78.01$\pm$0.62\\ 

\midrule
\multirow{1}{*}{\textbf{FedProx}}   & .128$\pm$.032 & 78.86$\pm$1.02 & .154$\pm$.032 & 80.43$\pm$0.74 & .126$\pm$.034 & 81.93$\pm$0.89 & .588$\pm$.005 & 78.05$\pm$0.63\\ 
\midrule
\multirow{1}{*}{\textbf{FedRep}}   & .164$\pm$.033 & 80.04$\pm$1.00 & .279$\pm$.029 & 80.08$\pm$0.01 & .214$\pm$.031 & 81.08$\pm$0.08 & .929$\pm$.004 & 57.01$\pm$0.48\\ 

\midrule
\multirow{1}{*}{\textbf{Ditto}}    & .124$\pm$.031 & 79.17$\pm$0.01 & .153$\pm$.032 & 80.41$\pm$0.74 & .128$\pm$.035 & 81.48$\pm$0.84 & .591$\pm$.005 & 78.05$\pm$0.63\\ 

\midrule
\multirow{1}{*}{\textbf{IFCA}} & .117$\pm$.031 & 77.89$\pm$0.95 & .140$\pm$.034 & 77.47$\pm$0.75 & .121$\pm$.034 & 80.41$\pm$0.87 & .063$\pm$.004 & 73.43$\pm$0.56\\ 

\multirow{1}{*}{\textbf{IFCA-S}} & .107$\pm$.032 & 78.89$\pm$0.96 & .134$\pm$.033 & 77.66$\pm$0.80 & .110$\pm$.035 & 81.51$\pm$0.90 & .061$\pm$.004 & 72.79$\pm$0.56\\ 

\midrule

\multirow{1}{*}{\textbf{FedSTL-S}} & .096$\pm$.026 & 81.67$\pm$0.93 & .148$\pm$.031 & 81.71$\pm$0.76 & .111$\pm$.030 & 83.44$\pm$0.87 & \textbf{.025$\pm$.003} & 77.03$\pm$0.93\\ 
                              
\multirow{1}{*}{\textbf{FedSTL}} & .095$\pm$.026 & 81.70$\pm$0.98 & .152$\pm$.031 & 81.83$\pm$0.71 & .119$\pm$.031 & 83.32$\pm$0.92 & .029$\pm$.003 &  78.99$\pm$0.66\\ 

\multirow{1}{*}{\textbf{FedSTL-T}} & \textbf{.076$\pm$.022} & \textbf{100.0$\pm$0.00} & \textbf{.118$\pm$.027} & \textbf{100.0$\pm$0.00} & \textbf{.099$\pm$.026} & \textbf{100.0$\pm$0.00} & .287$\pm$.013 & \textbf{100.0$\pm$0.00}\\
\bottomrule

\end{tabular}

\caption{Comparison on MSE and locally-distinctive property satisfaction.}
\vspace{-1em}
\label{tab:fhwares}
\end{table*}

\section{Evaluation}\label{sec:eval}
Our evaluations revolve around the following primary objectives:
(1) Enhancing personalized FL by incorporating \textit{locally-specific logic reasoning} properties into real-world sequential prediction datasets.
(2) Integrating \textit{intra-task symbolic reasoning} into multitask personalized FL training objectives and assessing its effectiveness in diverse client configurations.
(3) Highlighting the advantages of client model personalization enabled by our method, and comparing the results with other existing FL approaches.
The experiments were conducted on a machine equipped with an Intel Core i9-10850K CPU and an NVIDIA GeForce RTX 3070 GPU. The operating system used was Ubuntu 18.04.

\paragraph{Experiment Setup}
We evaluate the performance of FedSTL in two distinct scenarios: (1) a synthetic multivariate large-scale smart city dataset, and (2) a real-world univariate highway traffic volume dataset.
For baseline comparisons, we utilize three different backbone networks: a vanilla RNN, a GRU model, a transformer model, and an LSTM model. 
During each round of FL communication, we randomly select 10\% of the client devices to participate. For all the conducted experiments and algorithms, we use SGD with consistent learning rates and a batch size of 64.

\paragraph{Baseline Methods}
We compare the performance of FedSTL with the following methods: 
(1) \textit{FedAvg}~\cite{mcmahan2017communication} is a widely-used FL algorithm that trains a global model by aggregating the weighted average of client models;
(2) \textit{FedProx}~\cite{li2020federated} is a generalization of FedAvg that addresses system and statistical heterogeneity with a re-parametrization technique; 
(3) \textit{FedRep}~\cite{collins2021exploiting} is a personalized FL algorithm that learns shared global representations with unique local heads for each client; 
(4) \textit{Ditto}~\cite{li2021ditto} is a personalized FL method that incorporates regularization techniques to enhance the fairness and robustness;
(5) \textit{IFCA}~\cite{ghosh2020efficient} is a clustering FL algorithm that iteratively groups participating clients based on their training goals to promote collaboration among clients with similar objectives. 
% \begin{itemize}[wide=0pt, leftmargin=*, topsep=0pt, itemsep=0pt]
% \item \textit{FedAvg}\cite{mcmahan2017communication} is a widely-used FL algorithm that trains a global model by aggregating the weighted average of client models. 
% \item \textit{FedProx}\cite{li2020federated} is a generalization of FedAvg that addresses system and statistical heterogeneity in FL by introducing a re-parametrization technique. 
% \item \textit{FedRep}\cite{collins2021exploiting} is a personalized FL algorithm that learns shared global representations while allowing each client to have unique local heads. 
% \item \textit{Ditto}\cite{li2021ditto} is a personalized FL algorithm that incorporates regularization techniques to enhance fairness and robustness among FL clients.
% \item \textit{IFCA}~\cite{ghosh2020efficient} is a clustering FL algorithm that iteratively groups participating clients based on their training goals, thus promoting collaboration among clients with similar objectives. 
% \end{itemize}
% Please note that in order to ensure a fair comparison, we adopt the \textit{same} total number of epochs for each client across all methods. Specifically, 
We set the number of local epochs to 10 for FedAvg, FedProx, FedRep (with 8 head epochs), Ditto, and IFCA. Additionally, for FedSTL, we employ 6 local epochs and 4 cluster training epochs. 

\paragraph{Evaluation Metrics}
We utilize mean squared error (MSE) as our metric to evaluate the network performance. In addition, we introduce a measure called the satisfaction rate ($\rho \%$) to evaluate the impact of FedSTL on client property satisfaction. 
Specifically, we define $\rho \%$ as the percentage of network predictions, denoted by $\hat{\mathcal{Y}} = (y_{n+1}, \ldots, y_{n+m} )$, that satisfy a given property $\varphi$, induced by the input sequence $\mathcal{X} = ( x_1,\ldots, x_{n} )$. This allows us to quantify the degree to which the predicted sequence $\hat{\mathcal{Y}}$ adheres to the specified property $\varphi$ based on the input sequence $\mathcal{X}$.

\paragraph{Enhancing Locally-Specific Logic Reasoning Properties}
In our first task, we enhance personalized federated learning (FL) by \textit{incorporating locally-distinct temporal properties} using real highway traffic data. We specifically focus on the operational range property, as outlined in Table~\ref{tab:stl-case}, which captures important aspects of the traffic volume dynamics for each client during two-hour windows. 
We obtain a publicly available dataset from the Federal Highway Administration~\cite{hpmsfield} and preprocess hourly traffic volume from 15 states. 
Further, we design a testing scenario where a neural network is trained to predict the traffic volume for the next 24 consecutive hours based on the past traffic volume at a location over the previous five days.

\paragraph{Enhancing Intra-Task Symbolic Properties}\label{sec:multi}
In our second task, the objective is to enhance personalized federated learning (FL) by incorporating \textit{intra-task symbolic reasoning properties}, where the two variables were the number of vehicles on the road and the occupancy of the same road. To achieve this, we create a simulated dataset using SUMO (Simulation of Urban MObility) \cite{krajzewicz2002sumo}, a large-scale open-source road traffic simulator. The learning objective in this task is to predict a multivariate traffic and pollution scenario.
Moreover, we focus on diverse road types and consider a traffic scenario that includes cars, trucks, and motorcycles. From the available road segments, we select 100 segments to serve as FL clients. For each client, we record various features including vehicle counts, road occupancy, mean speed, carbon dioxide emission, average fuel consumption, and noise emission.

\paragraph{Results and Discussion}
Table~\ref{tab:fhwares} presents the results of FedSTL in enhancing locally-distinctive reasoning properties, where ``-S" indicates the evaluation on the cluster model, and ``-T" indicates the evaluation on the teacher model. The performance of FedSTL surpasses that of various other FL methods, both personalized and non-personalized. 
Specifically for FedSTL, the ``Cluster" row demonstrates the effectiveness of our clustering method, while the ``Client" row represents the performance on our client devices prior to prediction correction by the teacher. In contrast, the ``Teacher" row shows the framework's performance after the prediction is corrected by the teacher. Figure~\ref{fig:res} illustrates the comparison results on MSE by enhancing intra-task reasoning properties. The flowpipe representation is used to indicate the error bars. 

We observed a significant improvement in the MSE for RNN models, with up to a 54\% reduction compared to the baseline. Similarly, GRU models exhibited a 53.8\% lower MSE, and LSTM models achieved up to a 53.6\% lower MSE. Furthermore, the teacher model within the FedSTL framework consistently corrected predictions with a 100\% satisfaction rate across all cases.

By enhancing locally-distinct properties for FL clients, we observe a substantial improvement in the model's prediction performance, as indicated by both the MSE and satisfaction rate metrics.
When comparing FedAvg, FedProx, and Ditto, we find that their MSE values are generally similar. However, FedRep exhibits relatively poorer performance, while IFCA consistently outperforms the other methods. When locally-distinct properties are incorporated, FedSTL surpasses IFCA in terms of predictive accuracy. Notably, the teacher component of FedSTL demonstrates the best performance, with significantly lower MSE and a higher satisfaction rate for the properties.
These findings indicate a promising trend: by correcting predictions based on localized properties, we can achieve a significant improvement in the model's accuracy. Additionally, in the context of enhancing intra-task properties, both FedSTL and IFCA show superior performance compared to other baselines. These results underscore the importance of aligning client training objectives to enhance the overall performance of the model.

\begin{figure}[h]
\centering
    \includegraphics[width=0.4\textwidth]{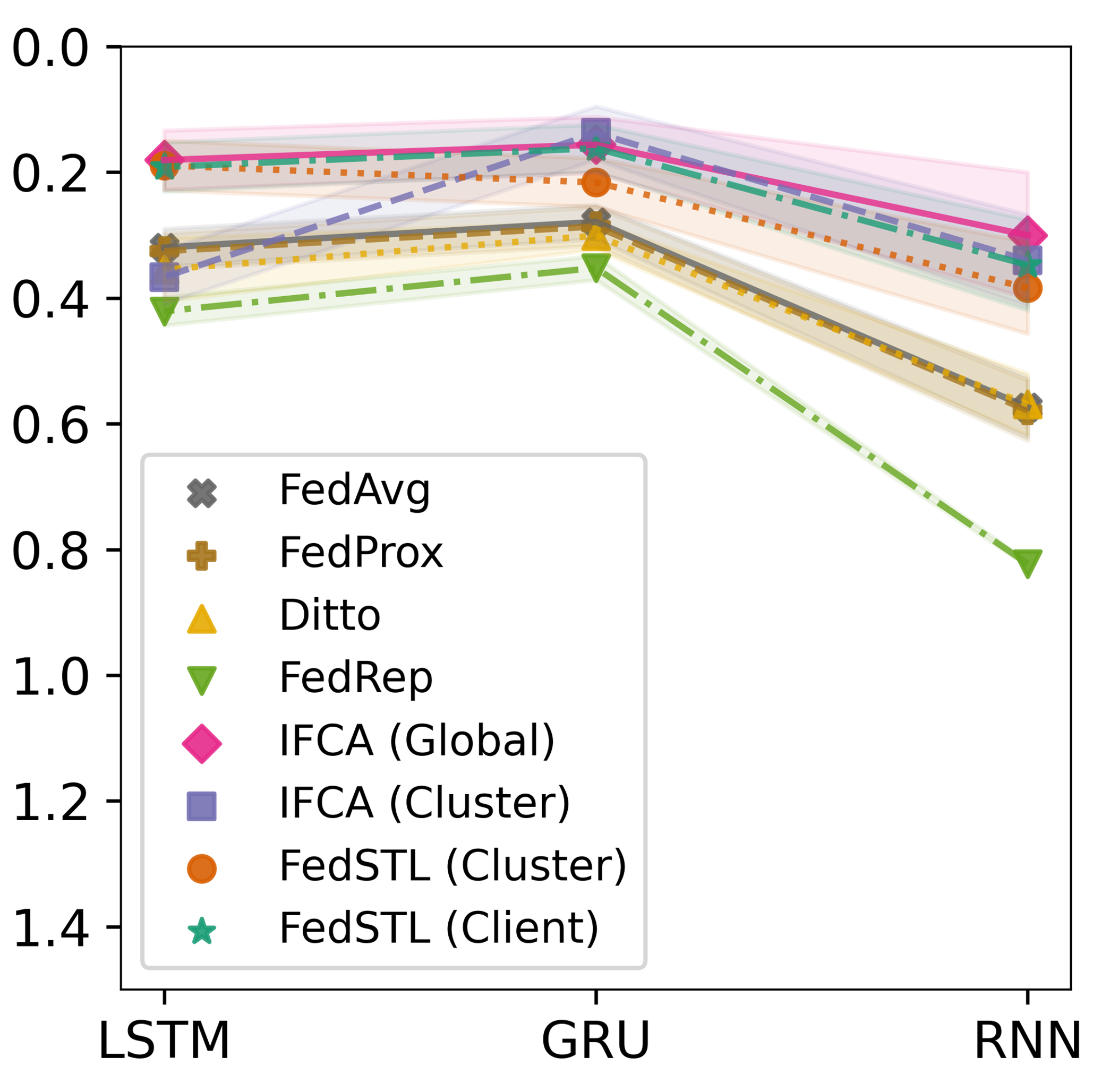}
    \caption{Comparison on MSE by enhancing intra-task reasoning properties. A higher position on the y-axis indicates a smaller MSE value.}
    \label{fig:res}
    \vspace{-1em}
\end{figure}
\section{Related Work}\label{sec:related}
% \meiyi{let's related work to the end}
% \subsection{Knowledge-Enhanced Learning}
In contrast to traditional FL frameworks, personalized FL prioritizes the training of local models tailored to individual clients, rather than relying on a single global model that performs similarly across all clients~\cite{tan2022towards, fallah2020personalized, collins2021exploiting,ghosh2020efficient,arivazhagan2019federated,mansour2020three}. 
Deng et al.~\cite{deng2020adaptive} highlight the significance of personalization in FL algorithms, particularly when dealing with non-i.i.d. client datasets. 
In light of this, our work focuses on investigating the potential benefits of incorporating symbolic reasoning through formal specification to enhance personalized FL algorithms. Specifically, we leverage rigorous and formal logic properties to improve the predictions of neural networks. This approach aligns with the concept of informed machine learning, which integrates auxiliary domain knowledge into the machine learning framework, as emphasized in a comprehensive survey by Von Rueden et al.~\cite{von2021informed}.  Generally, such methods are critical in improving the performance of data-driven models across various aspects, as shown in previous related works~\cite{muralidhar2018incorporating, ma2021predictive, diligenti2017integrating, jia2021physics,ma2020stlnet,hu2020harnessing,10.1145/3576841.3589633}. 

\section{Summary and Future Work}\label{sec:conclusion}
Previously, personalized FL methods have been developed to address the challenge of heterogeneous client devices. However, these methods have largely overlooked the potential of symbolic reasoning in tackling this issue. To bridge this gap, our study explores the effectiveness of incorporating symbolic reasoning into personalized FL. Our evaluation results demonstrate a significant improvement in both client prediction error and property satisfaction when leveraging induced client device properties. Furthermore, we observe an interesting phenomenon where promoting the satisfaction of desired properties also leads to a reduction in error rates. These promising findings highlight the benefits of equipping deep learning models with symbolic reasoning capabilities. 
% In future research, we intend to further investigate the advantages of symbolic reasoning-enhanced learning algorithms, including their robustness, fairness, and the ability to handle uncertainties. 

Moreover, while our primary focus was on evaluating prediction accuracy and property satisfaction, there is potential to extend this work to investigate the application of symbolic reasoning-enabled learning in a broader range of scenarios, such as the field of healthcare. 
Specifically, AI in healthcare faces unique challenges related to privacy, trust, and safety. 
Future research could explore how incorporating symbolic reasoning into learning algorithms can ensure trustworthy and safe AI-enabled predictions in healthcare settings.
Furthermore, by regulating model predictions with logic properties, post-hoc explainability can naturally emerge. 
Although our study is limited to evaluating prediction accuracy and property satisfaction, there is ample room for future research to explore the application of symbolic reasoning in other domains and harness the benefits, such as explainability, robustness, fairness, and the ability to handle uncertainties. 

\section*{Acknowledgments}
This material is based upon work supported by the National Science Foundation (NSF) under Award Numbers 2028001 and 2220401, AFOSR under FA9550-23-1-0135, and DARPA under FA8750-23-C-0518. 
%Any opinions, findings, and conclusions or recommendations expressed in this paper are those of the authors and do not necessarily reflect the views of AFOSR, DARPA, or NSF

% \section*{Broader Impact}\label{sec:broaderimp}
% Our proposed FL paradigm offers broad applicability beyond the evaluation cases presented in our experiments. It can be effectively applied to various distributed sequential prediction tasks within the FL framework. 
% Additionally, a key strength of our approach lies in its ability to accommodate specific logic reasoning properties through a teacher-student structure. This feature enables the enforcement of constraints on predictions, making our framework particularly suitable for safety-critical environments where satisfying certain requirements is crucial. 
% Furthermore, the logic-enhanced approach employed in our framework can be extended to handle 
% \textit{any} solvable templated STL formula. In other words, our framework will seamlessly integrate with any valid STL formula generated through STL specification mining algorithms, including those that are \textit{template-free}.

\bibliography{aaai24}

\begin{thebibliography}{41}
\providecommand{\natexlab}[1]{#1}

\bibitem[{An and Ma(2023)}]{10.1145/3576841.3589633}
An, Z.; and Ma, M. 2023.
\newblock Guiding Federated Learning with Inferenced Formal Logic Properties.
\newblock In \emph{Proceedings of the ACM/IEEE 14th International Conference on Cyber-Physical Systems (with CPS-IoT Week 2023)}, ICCPS '23, 274–275. New York, NY, USA: Association for Computing Machinery.
\newblock ISBN 9798400700361.

\bibitem[{Arivazhagan et~al.(2019)Arivazhagan, Aggarwal, Singh, and Choudhary}]{arivazhagan2019federated}
Arivazhagan, M.~G.; Aggarwal, V.; Singh, A.~K.; and Choudhary, S. 2019.
\newblock Federated learning with personalization layers.
\newblock \emph{arXiv preprint arXiv:1912.00818}.

\bibitem[{Bartocci et~al.(2022)Bartocci, Mateis, Nesterini, and Nickovic}]{bartocci2022survey}
Bartocci, E.; Mateis, C.; Nesterini, E.; and Nickovic, D. 2022.
\newblock Survey on mining signal temporal logic specifications.
\newblock \emph{Information and Computation}, 104957.

\bibitem[{Chen et~al.(2022{\natexlab{a}})Chen, Gao, Kuang, Li, and Ding}]{chen2022pfl}
Chen, D.; Gao, D.; Kuang, W.; Li, Y.; and Ding, B. 2022{\natexlab{a}}.
\newblock pFL-Bench: A Comprehensive Benchmark for Personalized Federated Learning.
\newblock \emph{arXiv preprint arXiv:2206.03655}.

\bibitem[{Chen et~al.(2022{\natexlab{b}})Chen, Gao, Kuang, Li, and Ding}]{NEURIPS2022_3cc03e19}
Chen, D.; Gao, D.; Kuang, W.; Li, Y.; and Ding, B. 2022{\natexlab{b}}.
\newblock pFL-Bench: A Comprehensive Benchmark for Personalized Federated Learning.
\newblock In Koyejo, S.; Mohamed, S.; Agarwal, A.; Belgrave, D.; Cho, K.; and Oh, A., eds., \emph{Advances in Neural Information Processing Systems}, volume~35, 9344--9360. Curran Associates, Inc.

\bibitem[{Collins et~al.(2021)Collins, Hassani, Mokhtari, and Shakkottai}]{collins2021exploiting}
Collins, L.; Hassani, H.; Mokhtari, A.; and Shakkottai, S. 2021.
\newblock Exploiting shared representations for personalized federated learning.
\newblock In \emph{International Conference on Machine Learning}, 2089--2099. PMLR.

\bibitem[{Deng, Kamani, and Mahdavi(2020)}]{deng2020adaptive}
Deng, Y.; Kamani, M.~M.; and Mahdavi, M. 2020.
\newblock Adaptive personalized federated learning.
\newblock \emph{arXiv preprint arXiv:2003.13461}.

\bibitem[{Diligenti, Roychowdhury, and Gori(2017)}]{diligenti2017integrating}
Diligenti, M.; Roychowdhury, S.; and Gori, M. 2017.
\newblock Integrating prior knowledge into deep learning.
\newblock In \emph{2017 16th IEEE international conference on machine learning and applications (ICMLA)}, 920--923. IEEE.

\bibitem[{Donz{\'e} and Maler(2010)}]{donze2010robust}
Donz{\'e}, A.; and Maler, O. 2010.
\newblock Robust satisfaction of temporal logic over real-valued signals.
\newblock In \emph{Formal Modeling and Analysis of Timed Systems: 8th International Conference, FORMATS 2010, Klosterneuburg, Austria, September 8-10, 2010. Proceedings 8}, 92--106. Springer.

\bibitem[{Fallah, Mokhtari, and Ozdaglar(2020)}]{fallah2020personalized}
Fallah, A.; Mokhtari, A.; and Ozdaglar, A. 2020.
\newblock Personalized federated learning: A meta-learning approach.
\newblock \emph{arXiv preprint arXiv:2002.07948}.

\bibitem[{FHWA(2016 [Online].)}]{hpmsfield}
FHWA. 2016 [Online].
\newblock Highway Performance Monitoring System Field Manual.
\newblock Office of Highway Policy Information.

\bibitem[{Ghosh et~al.(2020)Ghosh, Chung, Yin, and Ramchandran}]{ghosh2020efficient}
Ghosh, A.; Chung, J.; Yin, D.; and Ramchandran, K. 2020.
\newblock An efficient framework for clustered federated learning.
\newblock \emph{Advances in Neural Information Processing Systems}, 33: 19586--19597.

\bibitem[{Hinton, Vinyals, and Dean(2015)}]{hinton2015distilling}
Hinton, G.; Vinyals, O.; and Dean, J. 2015.
\newblock Distilling the Knowledge in a Neural Network.
\newblock arXiv:1503.02531.

\bibitem[{Hu et~al.(2020)Hu, Ma, Liu, Hovy, and Xing}]{hu2020harnessing}
Hu, Z.; Ma, X.; Liu, Z.; Hovy, E.; and Xing, E. 2020.
\newblock Harnessing Deep Neural Networks with Logic Rules.
\newblock arXiv:1603.06318.

\bibitem[{Jha et~al.(2017)Jha, Tiwari, Seshia, Sahai, and Shankar}]{jha2017telex}
Jha, S.; Tiwari, A.; Seshia, S.~A.; Sahai, T.; and Shankar, N. 2017.
\newblock Telex: Passive stl learning using only positive examples.
\newblock In \emph{International Conference on Runtime Verification}, 208--224. Springer.

\bibitem[{Jia et~al.(2021)Jia, Willard, Karpatne, Read, Zwart, Steinbach, and Kumar}]{jia2021physics}
Jia, X.; Willard, J.; Karpatne, A.; Read, J.~S.; Zwart, J.~A.; Steinbach, M.; and Kumar, V. 2021.
\newblock Physics-guided machine learning for scientific discovery: An application in simulating lake temperature profiles.
\newblock \emph{ACM/IMS Transactions on Data Science}, 2(3): 1--26.

\bibitem[{Karniadakis et~al.(2021)Karniadakis, Kevrekidis, Lu, Perdikaris, Wang, and Yang}]{karniadakis2021physics}
Karniadakis, G.~E.; Kevrekidis, I.~G.; Lu, L.; Perdikaris, P.; Wang, S.; and Yang, L. 2021.
\newblock Physics-informed machine learning.
\newblock \emph{Nature Reviews Physics}, 3(6): 422--440.

\bibitem[{Krajzewicz et~al.(2002)Krajzewicz, Hertkorn, R{\"o}ssel, and Wagner}]{krajzewicz2002sumo}
Krajzewicz, D.; Hertkorn, G.; R{\"o}ssel, C.; and Wagner, P. 2002.
\newblock SUMO (Simulation of Urban MObility)-an open-source traffic simulation.
\newblock In \emph{Proceedings of the 4th middle East Symposium on Simulation and Modelling (MESM20002)}, 183--187.

\bibitem[{Li et~al.(2020{\natexlab{a}})Li, Fan, Tse, and Lin}]{li2020review}
Li, L.; Fan, Y.; Tse, M.; and Lin, K.-Y. 2020{\natexlab{a}}.
\newblock A review of applications in federated learning.
\newblock \emph{Computers \& Industrial Engineering}, 149: 106854.

\bibitem[{Li et~al.(2021{\natexlab{a}})Li, Hu, Beirami, and Smith}]{li2021ditto}
Li, T.; Hu, S.; Beirami, A.; and Smith, V. 2021{\natexlab{a}}.
\newblock Ditto: Fair and robust federated learning through personalization.
\newblock In \emph{International Conference on Machine Learning}, 6357--6368. PMLR.

\bibitem[{Li et~al.(2020{\natexlab{b}})Li, Sahu, Talwalkar, and Smith}]{li2020federated}
Li, T.; Sahu, A.~K.; Talwalkar, A.; and Smith, V. 2020{\natexlab{b}}.
\newblock Federated learning: Challenges, methods, and future directions.
\newblock \emph{IEEE Signal Processing Magazine}, 37(3): 50--60.

\bibitem[{Li et~al.(2021{\natexlab{b}})Li, Jiang, Zhang, Kamp, and Dou}]{li2021fedbn}
Li, X.; Jiang, M.; Zhang, X.; Kamp, M.; and Dou, Q. 2021{\natexlab{b}}.
\newblock Fedbn: Federated learning on non-iid features via local batch normalization.
\newblock \emph{arXiv preprint arXiv:2102.07623}.

\bibitem[{Ma et~al.(2021{\natexlab{a}})Ma, Bartocci, Lifland, Stankovic, and Feng}]{ma2021novel}
Ma, M.; Bartocci, E.; Lifland, E.; Stankovic, J.~A.; and Feng, L. 2021{\natexlab{a}}.
\newblock A novel spatial--temporal specification-based monitoring system for smart cities.
\newblock \emph{IEEE Internet of Things Journal}, 8(15): 11793--11806.

\bibitem[{Ma et~al.(2020)Ma, Gao, Feng, and Stankovic}]{ma2020stlnet}
Ma, M.; Gao, J.; Feng, L.; and Stankovic, J. 2020.
\newblock STLnet: Signal temporal logic enforced multivariate recurrent neural networks.
\newblock \emph{Advances in Neural Information Processing Systems}, 33: 14604--14614.

\bibitem[{Ma et~al.(2019)Ma, Preum, Ahmed, T{\"a}rneberg, Hendawi, and Stankovic}]{ma2019data}
Ma, M.; Preum, S.~M.; Ahmed, M.~Y.; T{\"a}rneberg, W.; Hendawi, A.; and Stankovic, J.~A. 2019.
\newblock Data sets, modeling, and decision making in smart cities: A survey.
\newblock \emph{ACM Transactions on Cyber-Physical Systems}, 4(2): 1--28.

\bibitem[{Ma, Preum, and Stankovic(2017)}]{ma2017cityguard}
Ma, M.; Preum, S.~M.; and Stankovic, J.~A. 2017.
\newblock Cityguard: A watchdog for safety-aware conflict detection in smart cities.
\newblock In \emph{Proceedings of the Second International Conference on Internet-of-Things Design and Implementation}, 259--270.

\bibitem[{Ma et~al.(2021{\natexlab{b}})Ma, Stankovic, Bartocci, and Feng}]{ma2021predictive}
Ma, M.; Stankovic, J.; Bartocci, E.; and Feng, L. 2021{\natexlab{b}}.
\newblock Predictive monitoring with logic-calibrated uncertainty for cyber-physical systems.
\newblock \emph{ACM Transactions on Embedded Computing Systems (TECS)}, 20(5s): 1--25.

\bibitem[{Ma, Stankovic, and Feng(2018)}]{ma2018cityresolver}
Ma, M.; Stankovic, J.~A.; and Feng, L. 2018.
\newblock Cityresolver: a decision support system for conflict resolution in smart cities.
\newblock In \emph{2018 ACM/IEEE 9th International Conference on Cyber-Physical Systems (ICCPS)}, 55--64. IEEE.

\bibitem[{Ma, Stankovic, and Feng(2021)}]{ma2021toward}
Ma, M.; Stankovic, J.~A.; and Feng, L. 2021.
\newblock Toward formal methods for smart cities.
\newblock \emph{Computer}, 54(9): 39--48.

\bibitem[{Maler and Nickovic(2004)}]{maler2004monitoring}
Maler, O.; and Nickovic, D. 2004.
\newblock Monitoring temporal properties of continuous signals.
\newblock In \emph{Formal Techniques, Modelling and Analysis of Timed and Fault-Tolerant Systems}, 152--166. Springer.

\bibitem[{Mansour et~al.(2020)Mansour, Mohri, Ro, and Suresh}]{mansour2020three}
Mansour, Y.; Mohri, M.; Ro, J.; and Suresh, A.~T. 2020.
\newblock Three approaches for personalization with applications to federated learning.
\newblock \emph{arXiv preprint arXiv:2002.10619}.

\bibitem[{McMahan et~al.(2017)McMahan, Moore, Ramage, Hampson, and y~Arcas}]{mcmahan2017communication}
McMahan, B.; Moore, E.; Ramage, D.; Hampson, S.; and y~Arcas, B.~A. 2017.
\newblock Communication-efficient learning of deep networks from decentralized data.
\newblock In \emph{Artificial intelligence and statistics}, 1273--1282. PMLR.

\bibitem[{Mohri, Sivek, and Suresh(2019)}]{mohri2019agnostic}
Mohri, M.; Sivek, G.; and Suresh, A.~T. 2019.
\newblock Agnostic federated learning.
\newblock In \emph{International Conference on Machine Learning}, 4615--4625. PMLR.

\bibitem[{Muralidhar et~al.(2018)Muralidhar, Islam, Marwah, Karpatne, and Ramakrishnan}]{muralidhar2018incorporating}
Muralidhar, N.; Islam, M.~R.; Marwah, M.; Karpatne, A.; and Ramakrishnan, N. 2018.
\newblock Incorporating prior domain knowledge into deep neural networks.
\newblock In \emph{2018 IEEE international conference on big data (big data)}, 36--45. IEEE.

\bibitem[{Nguyen et~al.(2021)Nguyen, Ding, Pathirana, Seneviratne, Li, and Poor}]{nguyen2021federated}
Nguyen, D.~C.; Ding, M.; Pathirana, P.~N.; Seneviratne, A.; Li, J.; and Poor, H.~V. 2021.
\newblock Federated learning for internet of things: A comprehensive survey.
\newblock \emph{IEEE Communications Surveys \& Tutorials}, 23(3): 1622--1658.

\bibitem[{Preum et~al.(2021)Preum, Munir, Ma, Yasar, Stone, Williams, Alemzadeh, and Stankovic}]{preum2021review}
Preum, S.~M.; Munir, S.; Ma, M.; Yasar, M.~S.; Stone, D.~J.; Williams, R.; Alemzadeh, H.; and Stankovic, J.~A. 2021.
\newblock A review of cognitive assistants for healthcare: Trends, prospects, and future directions.
\newblock \emph{ACM Computing Surveys (CSUR)}, 53(6): 1--37.

\bibitem[{Shamir and Zhang(2013)}]{shamir2013stochastic}
Shamir, O.; and Zhang, T. 2013.
\newblock Stochastic gradient descent for non-smooth optimization: Convergence results and optimal averaging schemes.
\newblock In \emph{International conference on machine learning}, 71--79. PMLR.

\bibitem[{Tan et~al.(2022)Tan, Yu, Cui, and Yang}]{tan2022towards}
Tan, A.~Z.; Yu, H.; Cui, L.; and Yang, Q. 2022.
\newblock Towards personalized federated learning.
\newblock \emph{IEEE Transactions on Neural Networks and Learning Systems}.

\bibitem[{Von~Rueden et~al.(2021)Von~Rueden, Mayer, Beckh, Georgiev, Giesselbach, Heese, Kirsch, Pfrommer, Pick, Ramamurthy et~al.}]{von2021informed}
Von~Rueden, L.; Mayer, S.; Beckh, K.; Georgiev, B.; Giesselbach, S.; Heese, R.; Kirsch, B.; Pfrommer, J.; Pick, A.; Ramamurthy, R.; et~al. 2021.
\newblock Informed Machine Learning--A taxonomy and survey of integrating prior knowledge into learning systems.
\newblock \emph{IEEE Transactions on Knowledge and Data Engineering}, 35(1): 614--633.

\bibitem[{Yang et~al.(2019)Yang, Liu, Chen, and Tong}]{yang2019federated}
Yang, Q.; Liu, Y.; Chen, T.; and Tong, Y. 2019.
\newblock Federated machine learning: Concept and applications.
\newblock \emph{ACM Transactions on Intelligent Systems and Technology (TIST)}, 10(2): 1--19.

\bibitem[{Yu, Bagdasaryan, and Shmatikov(2020)}]{yu2020salvaging}
Yu, T.; Bagdasaryan, E.; and Shmatikov, V. 2020.
\newblock Salvaging federated learning by local adaptation.
\newblock \emph{arXiv preprint arXiv:2002.04758}.

\end{thebibliography}

\newpage
\clearpage
\appendix\markboth{Appendix}{Appendix}

\section*{Technical Appendix}

\section{Proof of Theorem}
In this section we briefly analyze the theory behind the proposed method. We first show that the local objective function $F_i$ in equation \ref{eq:localobj} is strongly convex with certain assumptions.

\begin{theorem}
    For any positive number $\lambda >0$, objective function $F_i(\theta_i) = \mathcal{L} (\mathcal{Y}, \hat{\mathcal{Y}}) + \lambda \, \mathcal{L}_{p} (\varphi_i, \hat{\mathcal{Y}})$ is strongly convex if the following conditions hold.
    
    \begin{enumerate}
        \item For all $i \in \mathcal{C}$, the STL properties $\varphi_i$ only includes $x \sim c$, $\land$ and $\square$.
        \item The original loss function $\mathcal{L}$ is strongly convex.
    \end{enumerate} 
    \label{theorem:convexity}
\end{theorem}

To prove this, we first show the additional STL loss is convex.
\begin{lemma}
When STL properties $\varphi_i, i \in \mathcal{C}$ only composes $\land$ and $\square$, $\mathcal{L}_p(\varphi_i, \hat{Y})$ is convex.
\end{lemma}
\begin{proof}
Consider the set  $S_{\varphi_i} = \left\{Y \mid Y \models  \varphi_i\right\}$ as the region of predictions that satisfies $\varphi_i$. We first show this set $S_{\varphi_i} $ is convex given condition. To show that, we use the induction method
\begin{itemize}
    \item Case $Y^{(k)} \sim c$: The region $\{Y \mid Y^{(k)} \geq c \}$ and $\{Y \mid Y^{(k)} \leq c \}$ are both half-planes, which are convex.
    \item Case $\varphi = \varphi_i \land \varphi_2$. The set $S_\varphi = \left\{Y \mid Y \models  \varphi\right\}$ is the intersection of two sets $S_{\varphi_1}$ and  $S_{\varphi_1}$. By induction,  $S_{\varphi_2}$  and $S_{\varphi_2}$ are both convex, which indicates $S_\varphi$ is also convex.
    \item Case $\varphi = \square_I \varphi_i$. Similar to the $\land$ case, the set $S_\varphi = \left\{Y \mid Y \models  \varphi\right\}$ is the intersection of sets $S_{\varphi_t}$ where $t\in I$. By induction,  
 each $S_{\varphi_t}$ is convex, which indicates the intersection set $S_\varphi$ is also convex.
\end{itemize}

Finally, $\mathcal{L}_p(\varphi_i, \hat{Y}) = \mathcal{L}_1(\hat{Y}, S_{\varphi_i} )$ is $\mathcal{L}_1$ distance to a convex set, which is also convex.
\end{proof}

Now we can then show the local objective function $F_i(\theta_i)$ is strong convex.
\begin{proof}
    As $\lambda > 0$ and $ \mathcal{L}_{p} (\varphi_i, \hat{\mathcal{Y}})$ is a  convex function, $\lambda \, \mathcal{L}_{p} (\varphi_i, \hat{\mathcal{Y}})$ is also convex.
    $F_i(\theta_i) = \mathcal{L} (\mathcal{Y}, \hat{\mathcal{Y}}) + \lambda \, \mathcal{L}_{p} (\varphi_i, \hat{\mathcal{Y}})$ is a summation of a strongly convex function $\mathcal{L} (\mathcal{Y}, \hat{\mathcal{Y}})$ and a convex  function $\lambda \, \mathcal{L}_{p} (\varphi_i, \hat{\mathcal{Y}})$. By the definition of strong convexity, $F_i(\theta_i)$ is strongly convex.
\end{proof}

We can then analyze the convergence of FedSTL. Because our loss function is strongly convex, a well known result is that SGD has convergence on strong convex function\cite{shamir2013stochastic}. Specifically,

\begin{theorem}[Convergence of stochastic gradient descent\cite{shamir2013stochastic}]
For any strongly convex function $f(x)$, the expected function value $\mathbf{E}[f_T(x)]$ after performed SGD for $T$ steps satisfies $$ \mathbf{E}\left[f_T(x)-f^*\right] \leq C_1 \cdot \frac{1+\log(T)}{T},$$ where $C_1$ is a constant that depends on the function and $f^*$ is the global minimal function value.
\label{theorem:sgd_convergence}
\end{theorem}

We then formally prove the convergence of a client.
\begin{lemma}
With the assumptions in Theorem \ref{theorem:convexity} and further assume $\mathcal{L}$ is realizable (that is, $\exists\theta_i = \theta_i^*$ that $\mathcal{L} (\mathcal{Y}, \hat{\mathcal{Y}}) = 0$), for any $\epsilon > 0$, there exists a $\tau>0$ that after $\tau$ steps of SGD update, we have the expected original loss of the client $\mathbf{E}[\mathcal{L} (\mathcal{Y}, \hat{\mathcal{Y}})] \leq \epsilon$. In other words, the loss can be arbitrarily small.
\end{lemma}

\begin{proof}
    By Theorem \ref{theorem:convexity}, objective $F_i(\theta_i)= \mathcal{L} (\mathcal{Y}, \hat{\mathcal{Y}}) + \lambda \cdot \mathcal{L}_{p} (\varphi_i, \hat{\mathcal{Y}})$ is strongly convex.
    
    According to Theorem \ref{theorem:sgd_convergence}, let $$\tau = {(\frac{2C_1}{\epsilon})}^2,$$
    We have $$\mathbf{E}[F_i(\theta_i) - F_i(\theta_i^*) ] \leq C_1 \cdot \frac{1 + log(\tau)}{\tau} < C_1 \cdot \frac{2\sqrt{\tau}}{\tau}=\epsilon.$$

    Furthermore, as the dataset satisfies the property defined in $\phi_j$, $L_p(\varphi_i, \hat{Y}) = 0$ for $\theta^*_i$. Because the original objective is  Therefore, $F_i(\theta_i^*) = 0$. We then have $$\epsilon > \mathbf{E}[F_i(\theta_i) - F_i(\theta_i^*)] = \mathbf{E}[F_i(\theta_i)] \geq \mathbf{E}[\mathcal{L} (\mathcal{Y}, \hat{\mathcal{Y}})], $$ which finishes the proof.
\end{proof}

% In addition, in each clustering step, the total loss over all client models will only decrease. ----- Remove this claim because probably does not hold if properties change a lot?

Note that while our assumption in theoretical analysis is rather strong, non-convex functions (such as deep neural networks) have already achieved good performance through optimization in real practice. According to our experimental analysis, general STL properties can also be used in our proposed algorithm and support the model training.

\section{Signal Temporal Logic and Its Inference}

In this section, we discuss additional details on STL and the STL-based property inference task. 
We begin by providing definitions of STL qualitative semantics~\cite{maler2004monitoring}, quantitative semantics~\cite{donze2010robust}, and additional formulas used in the specification mining task in FedSTL. We present illustrative examples to demonstrate these concepts in practical scenarios.

To begin, Definition~\ref{def:stlbool} presents the qualitative (Boolean) semantics for an STL formula $\varphi$. Here, we use the following notations: $\mathbf{x}$ denotes a signal trace, $\mu$ denotes an STL predicate, and $\varphi, \varphi_1$, and $\varphi_2$ represent different STL formulas.

\begin{definition}[STL Qualitative Semantics]\label{def:stlbool}
\begin{align*}
& \left(\mathbf{x}, t\right) \vDash \top & \Leftrightarrow & \quad \top \\
& \left(\mathbf{x}, t\right) \vDash \mu & \Leftrightarrow & \quad \mu \left(\mathbf{x} [t] \right) \\
& \left(\mathbf{x}, t\right) \vDash \varphi_1 \vee \varphi_2 \quad & \Leftrightarrow & \quad\left(\mathbf{x}, t\right) \vDash \varphi_1 \vee \left(\mathbf{x}, t\right) \vDash \varphi_2\\
& \left(\mathbf{x}, t\right) \vDash \varphi_1 \wedge \varphi_2 \quad & \Leftrightarrow & \quad\left(\mathbf{x}, t\right) \vDash \varphi_1 \wedge\left(\mathbf{x}, t\right) \vDash \varphi_2\\
& \left(\mathbf{x}, t\right) \vDash \Diamond_{[a, b]} \varphi \quad & \Leftrightarrow & \quad \exists t^{\prime} \in\left[t+a, t+b\right],\left(\mathbf{x}, t^{\prime}\right) \vDash \varphi\\
& \left(\mathbf{x}, t\right) \vDash \square_{[a, b]} \varphi \quad & \Leftrightarrow & \quad \forall t^{\prime} \in\left[t+a, t+b\right], \left(\mathbf{x}, t^{\prime}\right) \vDash \varphi\\
& \left(\mathbf{x}, t\right) \vDash \varphi_1 \mathcal{U}_{[a, b]} \varphi_2 & \Leftrightarrow & \quad\exists t^{\prime} \in\left[t+a, t+b\right], \left(\mathbf{x}, t^{\prime}\right) \vDash \varphi_2
\\&&& \quad \wedge \forall t^{\prime \prime} \in\left[t, t^{\prime}\right],\left(\mathbf{x}, t^{\prime \prime}\right) \vDash \varphi_1
\end{align*}
\end{definition}

While the qualitative semantics provide a Boolean evaluation of the satisfaction of an STL property, our specification inference task relies on a real-valued measurement of property satisfaction, known as the STL robustness metric ($\rho$). Definition~\ref{def:stlrob} describes how this metric is calculated, which maps a given signal trace $\mathbf{x}$ and an STL formula $\varphi$ to a real number over a specified time interval $I$. 

\begin{definition}[STL Robustness $\rho$]\label{def:stlrob}
\begin{align*}
&\rho(x \sim c, \varphi, t) &&= \pi(\mathbf{x}[t] ) - c \\
&\rho(\neg \varphi, \mathbf{x}, t) &&= - \rho(\varphi, \mathbf{x}, t) \\
&\rho(\varphi_1 \vee \varphi_2, \mathbf{x}, t) &&=  \max\{\rho(\varphi_1, \mathbf{x}, t), \rho(\varphi_2, \mathbf{x}, t) \}\\
&\rho(\varphi_1 \land \varphi_2, \mathbf{x}, t) &&=  \min\{\rho(\varphi_1, \mathbf{x}, t), \rho(\varphi_2, \mathbf{x}, t) \}\\
& \rho(\Diamond_I \varphi , \mathbf{x}, t) && = \underset{t' \in (t, t+I)}{\max} \rho(\varphi, \mathbf{x}, t')\\
& \rho(\square_I \varphi, \mathbf{x}, t) && = \underset{t' \in (t, t+I)}{\min} \rho(\varphi, \mathbf{x}, t') \\
&\rho(\varphi_1 \mathcal{U} \varphi_2, \mathbf{x}, t) &&= \sup_{t'\in (t + I) \cap \mathbb{T}} (\min\{\rho(\varphi_2, \mathbf{x}, t'), \\
&&& \, \quad \inf_{t''\in[t,t']}(\rho(\varphi_1, \mathbf{x}, t'')) \})
\end{align*}
\end{definition}

\subsection{Example of STL specification and monitoring}
To demonstrate the application of STL, we present a simple example. The trajectory illustrated below is generated using the function $x(t)=sin(t)$, representing a sine wave with an amplitude and frequency both set to 1. The STL formula  $\varphi_1 = \Diamond_{[0, 2\pi]} ( x \geq 1 )) \wedge \Diamond_{[2\pi, 4\pi]} ( x \geq 1 )) \wedge \Diamond_{[4\pi, 6\pi]} ( x \geq 1 ))$ can be used to express the requirement that during the time intervals $[0,2\pi] $, $[2\pi, 4\pi] $, and $[4\pi, 6\pi]$, the trajectory must be greater than or equal to 1 at some point. Moreover, the STL formula $\varphi_2 = (x \geq 0) \; \mathcal{U}_{[0,2\pi]} \, (0 > x)$ can be used to express the requirement that during the time interval $[0,2\pi]$, the trajectory must be greater than or equal to zero, unless at some point it becomes negative.
% \begin{figure}[H]
%     \centering
%     \includegraphics[width=0.45\textwidth]{figs/example.png}
%     \caption{Trajectory generated with the function $x(t)=sin(t)$ that represents a simple sine wave.}
%     \label{fig:example}
% \end{figure}

\subsection{STL robustness-based logic property inference}
As described in Section 3, the goal of STL property inference is to derive a full STL formula that closely matches a given set of sequences. To achieve this, let $\varphi'$ be the inferred formula from sequence $\mathbf{x}$, and we aim for the robustness $\rho(\varphi',\mathbf{x})$ to be a small positive number ($0 < \rho \ll 1$). By utilizing the STL robustness metric introduced earlier, we can formulate the objective of learning an STL requirement as Equation~\ref{eq:solverobu}, where $\alpha$ and $\alpha^\prime$ represent two candidate values of the unknown variable. 

\begin{equation}\label{eq:solverobu}
\begin{gathered}
    \mathop{\text{min}}{|\epsilon|} \; s.t. \; \epsilon=\alpha^\prime-\alpha  \\ 
    \text{where} \; \rho(\varphi(\alpha),\mathbf{x},t) \geq 0 \; \text{and} \; \rho(\varphi(\alpha^\prime),\mathbf{x},t) \leq 0
\end{gathered}
\end{equation}

In other words, our objective is to minimize the value of $\epsilon$, which represents the discrepancy between a satisfactory parameter value and an unsatisfactory parameter value within the context of the STL robustness function. However, the STL robustness function is not differentiable at zero. Therefore, to tackle this issue, Equation 1 is utilized as an approximation to effectively address this problem.

\section{Additional Real-World Use Cases}\label{app:b}

In Section 3.2, we present eight types of properties that are enhanced by FedSTL. In this section, we further demonstrate the versatility of FedSTL by providing additional logic reasoning specifications. These examples illustrate how our framework can be seamlessly applied to various real-world scenarios, thus showcasing its generalizability.

\paragraph{Smart Home}
Smart Home refers to an integrated IoT technology that connects household appliances to the internet. By leveraging artificial intelligence (AI), Smart Home solutions offer enhanced customization capabilities for user profiles. One notable advantage of employing federated learning (FL) algorithms in the context of Smart Home is the mitigation of privacy leakage concerns.

\begin{itemize}[wide=10pt, leftmargin=*, topsep=0pt]
    \item \textit{Operational Range:} The usage duration of a smart home lighting system within each hour of a day should be limited to a maximum of $a_i$ minutes and a minimum of $b_i$ minutes. 
    \[ \bigwedge_{i=1}^{1, 2, ..., 24} (\square_{[i,i+t]} (x\leq a_i \wedge x\geq b_i)) \]
    \item \textit{Existence:} The automated pet feeder should be activated (denoted as $act$) a minimum of $a$ times per day, and the smart lawn watering system (denoted as $lawn$) should be activated at least $b$ times per day. 
    \[ \Diamond_{[1,24]} (\textit{act} > a) \wedge \Diamond_{[1,24]} (\textit{lawn} > b) \]
\end{itemize}

\paragraph{Smart Grid}
Smart Grid enables bidirectional communication between power plants, factories, city offices, and homes. In the context of Smart Grid services, Federated Learning (FL) is applied in conjunction with smart meters and devices to analyze energy consumption, mitigating concerns about privacy breaches. 

\begin{itemize}[wide=10pt, leftmargin=*, topsep=0pt]
    \item \textit{Operational Range:} The demand for electricity in regular households typically exhibits higher demand occurring between 4-9 p.m. during peak hours, while the demand is lowest during the night hours. 
    \[ (\square_{[16,21]} \, \textit{demand} > d_{higher}) \wedge (\square_{[0,5]} d_{lower} > \textit{demand}) \]
    \item \textit{Temporal Implications:} The power consumption should decrease during a severe storm that causes a power outage. 
    \[ \square_{[0,24]} \, (\textit{outage} \rightarrow \Diamond_{[0,10]} \, \textit{power} \leq previous(\textit{power}) ) \]
    \item \textit{Intra-Task:} Households equipped with solar panels that supplement electricity consumption exhibit lower overall electric energy consumption $c$, particularly during the daytime. 
    
    $ \square_{[6,17]} \, (\textit{solar panel} \wedge \textit{irradiance} \geq 800 )\rightarrow (\Diamond_{[0,5]  \allowbreak 
    } \lvert  c_2 - c_1 \rvert > 20 ) $
\end{itemize}
    
\paragraph{Smart Healthcare}
Smart Health technologies offer significant advantages to modern healthcare systems, such as remote health monitoring for patients with chronic diseases. By leveraging FL, these systems improve patient monitoring, enable early detection of health issues, and facilitate personalized interventions through secure analysis of data from diverse sources without compromising privacy. 

\begin{itemize}[wide=10pt, leftmargin=*, topsep=0pt]
    \item \textit{Operational Range:} In a human activity recognition task, the effectiveness of an exercise is ensured by monitoring whether the linear and angular acceleration are within a specific range. 
    \[ \square_{[0,t]} ( a_{min} \leq a \leq a_{max} \wedge  \alpha_{min} \leq \alpha \leq \alpha_{max} ) \]
    \item \textit{Until:} In a repeated shoulder abduction task, the horizontal movement $\theta$ of the patient's shoulder must be confined to a narrow range until a certain vertical angle $\varphi$ is reached. 
    \[ \square_{[0,t]} ( 0 \leq \theta \leq \theta_{max} ) \; \mathcal{U}_{[0,\tau]} \, ( \varphi_{min} \leq \varphi ) \]
\end{itemize}

\paragraph{Autonomous Driving}
FL techniques have emerged to address privacy concerns in vehicular edge computing networks. In line with this, we show three ways in which STL properties can be extended to autonomous driving tasks. 
\begin{itemize}[wide=10pt, leftmargin=*, topsep=0pt]
    \item \textit{Operational Range:} The ego vehicle $p_i$ must maintain a safe braking distance from the ado vehicle $p_j$. \[ \square_{[0,t]} (\lvert p_i - p_j \rvert \geq d_{safe}) \] 
    \item \textit{Existence:} The ego vehicle $i$ should eventually reach the target speed. 
    \[ \Diamond_{[0,t]} v_i \geq v_{target} \]
    \item \textit{Multiple Eventualities: } The ego vehicle should sequentially reach three destinations. 
    
    $ \square_{[0,t_1]} ( \lvert p_i - p_{goal}^1 \rvert \geq d_{min} ) \wedge \square_{[t_1,t_2]} ( \lvert p_i - p_{goal}^2 \rvert \allowbreak \geq d_{min} ) \wedge \square_{[t_2,t_3]} ( \lvert p_i - p_{goal}^3 \rvert \geq d_{min} ) $
\end{itemize}

% \begin{table}[h]\caption{STL templates for additional applications}
% \scriptsize
% \centering
% \begin{tabular}{|p{3.5cm}|p{4.2cm}|}
% \hline
% Property     & STL template\\
% \hline
% \begin{tabular}[x]{@{}l@{}} 1: \textbf{Temporal Bound}: Traffic volume
% \\ is upper-bounded by threshold $a$ 
% \\ and lower-bounded by threshold $b$
% \end{tabular}
% &
% $\begin{aligned}
% & \bigwedge_{i=1}^{1, 2, ..., t} (\square_{[i,i+1]} (x\leq a_i\wedge x\geq b_i) )
% \end{aligned}$ \\
% \hline
% \begin{tabular}[x]{@{}l@{}} 2: \textbf{Existence}: Traffic volume
% \\ should eventually reach the upper
% \\ extreme $a$ and the lower extreme $b$
% \end{tabular}
% &
% $\begin{aligned}
% & \bigwedge_{i=1}^{1, 2, ..., t} (\Diamond_{[i,i+1]} (x \leq a_i \wedge x \geq b_i) )
% \end{aligned}$ \\
% \hline
% \begin{tabular}[x]{@{}l@{}} 3: \textbf{Multivariate}: The difference 
% \\ between road occupancy and vehicle
% \\ count should be greater than $a$
% \end{tabular}
% &
% $\begin{aligned}
% & \bigwedge_{i=1}^{1, 2, ..., t} (\square_{[i,i+1]}(\{x_{occ}-x_{ct}\}>a)
% \end{aligned}$ \\
% \hline
% \end{tabular}
% \label{tab:stl-app}
% \end{table}

\section{Additional Experimental Details}

\subsection{SUMO Dataset}
Simulation of Urban MObility (SUMO)~\cite{krajzewicz2002sumo} is an open-source road traffic simulator widely used for multi-modal traffic and smart city research. In this study, we utilize a SUMO-simulated dataset to evaluate the performance of FedSTL in a multivariate smart city and pollution prediction task based in Nashville, TN. 
To construct the dataset, we focus on primary and secondary road types for simplicity. We consider a large-scale road network that accommodates cars, trucks, and motorcycles, while public transportation modes such as trains, buses, bicycles, and pedestrians are planned for future work.
For the evaluation, we select 100 road segments out of a total of 20,309 as FL clients. We record various parameters for each client, including vehicle counts, road occupancy, mean speed, carbon dioxide emission, average fuel consumption, and noise emission, for 3,600 simulation steps.
% Figure~\ref{fig:map} depicts the simulated geographical regions near Nashville, TN. The left plot illustrates the converted road network in SUMO, while the right plot presents the corresponding map.

% \begin{figure}[t]
%     \centering
%     \includegraphics[width=.45\textwidth]{figs/map_convert.png}
%     \caption{In the left plot, the colors are calibrated to represent the speed limit of each road segment. Red indicates a speed limit of 5 mph, yellow indicates 15 mph, green indicates 20 mph, light blue indicates 30 mph, dark blue indicates 40 mph, and purple indicates a speed limit of 60 mph and above. This color scheme provides a visual representation of the varying speed limits across the road network.}
%     \label{fig:map}
% \end{figure}

\begin{table*}[t]
\caption{Runtime Evaluation of Specification Mining with Various Templates on 500 Clients.}
\small
\centering
\begin{adjustbox}{max width=\textwidth}
\begin{tabular}{|l|l|l|}
\hline
\textbf{Temporal Reasoning Property} & \textbf{Templated Logic Formula} & \textbf{Runtime (s)} \\ \hline
1. Operational Range (single)           &   $\begin{aligned}
& \square_{[i,i+t]} (x\leq a_i\wedge x\geq b_i) 
\end{aligned}$                       &   0.0032      \\ \hline
2. Existence (single)                   &   $\begin{aligned}
& \Diamond_{[i,i+t]} (x \leq a_i \wedge x \geq b_i) 
\end{aligned}$                      &  0.0052       \\ \hline
3. Until (single)                       &   $\begin{aligned}
& (x < a_i) \, \mathcal{U}_{[i,i+1]} (x < b_i)
\end{aligned}$                       &  0.0111       \\ \hline
4. Intra-task Reasoning        &   $\begin{aligned}
& \square_{[i,i+t]}((x_1-x_2)>a_i)
\end{aligned}$                       &  0.0065       \\ \hline
5. Temporal Implications       &   $\begin{aligned}
& \square_{[t_1,t_2]} ((x \geq a_1) \rightarrow \Diamond_{[t_3,t_4]}(x \geq a_2)) 
\end{aligned}$                       &  0.0239       \\ \hline
6. Intra-task Nested Reasoning &   $\begin{aligned}
& \square_{[t_1,t_2]}((x_1 \geq a) \rightarrow \Diamond_{[t_3,t_4]} (x_2 \geq b))
\end{aligned}$                       &  0.0251       \\ \hline
7. Multiple Eventualities      &   $\begin{aligned}
& \Diamond_{[t_1,t_2]} (x \geq a_1) \wedge \cdots \wedge \Diamond_{[t_3,t_4]} (x \geq a_n)
\end{aligned}$                       &   0.0274      \\ \hline
\end{tabular}
\end{adjustbox}
\label{tab:p-t}
\end{table*}

\begin{table}[t]
\caption{Specification Mining Runtime(s) Across Varied Templates on Average for 500 Clients in Seconds.}
\small
\centering
\begin{adjustbox}{max width=0.45\textwidth}
\begin{tabular}{|l|c|c|c|}
\hline
\textbf{Temporal Reasoning Property} & \textbf{5 formulas} & \textbf{10 formulas} & \textbf{15 formulas} \\ \hline
1. Operational Range           & 0.0166   &  0.0317    &  0.0469    \\ \hline
2. Existence                   & 0.0253  &  0.0521    &  0.0769    \\ \hline
3. Until                       & 0.0532  &  0.1101    &  0.1607    \\ \hline
4. Intra-task Reasoning        & 0.0327  &  0.0647    &  0.0966    \\ \hline
5. Temporal Implications       & 0.1216  &  0.2443   &  0.3588  \\ \hline
6. Intra-task Nested Reasoning & 0.1224  &  0.2421   &  0.3684   \\ \hline
7. Multiple Eventualities      & 0.1390  &  0.2832   &  0.4173   \\ \hline
\end{tabular}
\end{adjustbox}
\label{tab:nof-t}
\end{table}

\begin{table}[t]
\caption{Average Cluster Training Time for 500 Clients in Seconds.}
\small
\centering
\begin{adjustbox}{max width=0.45\textwidth}
\begin{tabular}{|c|c|c|c|c|}
\hline
\textbf{Clients Per Cluster} & \textbf{5 formulas} & \textbf{10 formulas} & \textbf{15 formulas} & \textbf{20 formulas} \\ \hline
10            &   0.5545      &  0.6801  &  0.8035    &  0.9456    \\ \hline
20            &   0.6970      &  0.8422  &  1.0038    &  1.2006    \\ \hline
50            &   1.0387      &  1.2457  &  1.5123    &  1.7981    \\ \hline
100           &   1.1792      &  1.4327  &  1.7991    &  2.0559    \\ \hline
\end{tabular}
\end{adjustbox}
\label{tab:training-t}
\end{table}

The selection of road segments is imported from OpenStreetMap (OSM) using the \texttt{OSMWebWizard} tool provided by SUMO. The tool allows us to import specific road segments based on our requirements.
To generate the traffic demand, we utilize the SUMO simulator. The data is retrieved through the TraCI Python interface, which provides a convenient way to interact with the simulation. We use the following commands to retrieve the necessary traffic data:
\begin{itemize}
    \item \texttt{traci.edge.getLastStepVehicleNumber()}
    \item \texttt{traci.edge.getLastStepOccupancy()}
    \item \texttt{traci.edge.getLastStepMeanSpeed()}
    \item \texttt{traci.edge.getCO2Emission()}
    \item \texttt{traci.edge.getFuelConsumption()}
    \item \texttt{traci.edge.getNoiseEmission()} 
\end{itemize}

\subsection{Dataset Format}
The real-world traffic dataset used in our study is obtained from the U.S. Traffic Volume Data Tabulations. We specifically acquire the hourly traffic volume data for the year 2019 from 15 different states. The data is initially stored in plain text format.
Before using the data for our experiments, we preprocess it to handle missing values. Specifically, we apply linear interpolation to fill in any gaps in the data. The preprocessed dataset is then saved in the \textit{.npy} format, which is a binary file format for storing numpy arrays.
To ensure consistency in our evaluations, we split the dataset into three parts: training, evaluation, and testing. We use an 80/10/10 split ratio, where 80\% of the data is used for training, 10\% for evaluation, and 10\% for testing.

\subsection{Additional Evaluations on Runtime}

Our proposed framework is designed with scalability as a central consideration. 
Specifically, during the client updating process, each client generates their properties individually offline, and each client also follows a parallel model training scheme. Similarly, cluster models adhere to a similar design, ensuring that the models are trained individually in parallel. Therefore, the total training time of each framework would not increase with more participating clients. To further demonstrate the scalability of our framework and evaluate the runtime at a larger scale, we conducted additional evaluations involving 500 federated learning clients. 

Table~\ref{tab:p-t} shows the average time required for extracting various types of logic reasoning properties across a dataset of 500 clients. The provided runtimes represent the average time taken to extract each of the logic properties as outlined in Table~\ref{tab:stl-case} of the paper. Notably, for straightforward yet expressive properties like operational range and existence, the average extraction time for one property on a time series with a length of 180 can be as short as 0.0032 seconds. Conversely, the most time-consuming property to extract is multiple eventualities, primarily due to the complexity of the multiple $\wedge$ operations involved.

In Table~\ref{tab:p-t}, we showcase the time required for extracting various types of logic reasoning properties, averaged across 500 clients. By enumerating the number of property formulas per client, we generally observe a linear trend in the time increment. For the majority of FL clients in real-world applications, as mentioned in the previous sections, having 15 formulas is expressive enough to specify the important properties. As evident from the table, even when extracting all 15 properties combined, it still takes less than half a second to complete this process. 

Finally, in Table~\ref{tab:training-t}, we provide a demonstration of the average training time required per cluster for each communication round. Specifically, we set the number of local updates for each client to be 10 and the number of cluster training epochs to be 5. The first column indicates the number of clients per cluster. In real-world scenarios, different cluster sizes can influence model distributions, aggregation time, and property inference time. To calculate the average training time per cluster, we aggregate the average client training time for that cluster, the model aggregation time, and the averaging cluster training time. As depicted in the table, the aggregated training time per cluster per communication round increases with the number of formulas and the number of clients belonging to that cluster. In real model deployments, each cluster runs in parallel with the others, further enhancing the scalability of FedSTL.

% \subsection{Additional Results}
% In this section, we provide additional details on the training curve (Figure~\ref{fig:curve}) for the purpose of comparing baseline methods. 
% The training curve is calculated once all the `client' models have completed their training. In the case of FedSTL, cluster training is performed after the individual client training. As a result, we observe that the training and evaluation loss is higher compared to IFCA, which conducts all the training simultaneously. While this difference is expected, it is worth noting that both IFCA and FedSTL outperform the other baseline methods, particularly in the case of \textit{Transformer} models. 

% \begin{figure}[h]
%   \centering
%   \includegraphics[width=0.45\textwidth]{figs/curve.png}
%   \caption{Comparison of \textit{client} training curves for different FL methods on Transformer, RNN, GRU, and LSTM.}
%   \label{fig:curve}
% \end{figure}

\subsection{Experiment Parameters}

\paragraph{General FL settings} In each testing scenario, we have a fixed number of 100 FL clients. During the communication rounds, 10\% of the clients are randomly selected.
The batch size is set to be 64. 
For the Transformer, RNN, and GRU models, we set the maximum learning rate to be 0.001. However, for the LSTM models, the maximum learning rate is set to 0.01. 
To optimize the models, we utilize the SGD optimizer in the PyTorch implementation. The optimizer's `weight\_decay' parameter is set to 0.0001. We use the mean squared error (MSE) loss, implemented in PyTorch, which is commonly employed in regression tasks. In our default setup, the number of local training epochs is set to be 10 if not otherwise specified. 

\paragraph{Additional Parameters}
In FedProx, we set the value of the parameter $\mu$ to be 0.1, which controls the balance between the local and global models during aggregation. In Ditto, the parameter $\lambda$ is set to 1. In both IFCA and FedSTL, the number of clusters is set to be 5, determining the division of clients into distinct groups for clustering-based FL. In FedRep, the epochs for training local heads is set to be 8, indicating the number of iterations each client conducts on its local data to update its private head. In FedSTL, we modify the number of local training epochs to be 8, specifying the number of iterations each client performs on its local data. Additionally, we set the cluster training epochs to be 2. The following layers are designated as the publicly aggregated layers for both FedRep and FedSTL.  

\begin{itemize}[wide=10pt, leftmargin=*, topsep=0pt]
    \item RNN
    \begin{itemize}[wide=10pt, leftmargin=*, topsep=0pt]
        \item \texttt{rnn\_1.weight\_ih}
        \item \texttt{rnn\_1.weight\_hh}
        \item \texttt{rnn\_1.bias\_ih}
        \item \texttt{rnn\_1.weight\_hh}
        \item \texttt{rnn\_1.bias\_hh}
        \item \texttt{rnn\_2.weight\_ih}
        \item \texttt{rnn\_2.weight\_hh}
        \item \texttt{rnn\_2.bias\_ih}
        \item \texttt{rnn\_2.weight\_hh}
        \item \texttt{rnn\_2.bias\_hh}
    \end{itemize}
    \item GRU
    \begin{itemize}[wide=10pt, leftmargin=*, topsep=0pt]
        \item \texttt{gru\_1.weight\_ih}
        \item \texttt{gru\_1.weight\_hh}
        \item \texttt{gru\_1.bias\_ih}
        \item \texttt{gru\_1.weight\_hh}
        \item \texttt{gru\_1.bias\_hh}
        \item \texttt{gru\_2.weight\_ih}
        \item \texttt{gru\_2.weight\_hh}
        \item \texttt{gru\_2.bias\_ih}
        \item \texttt{gru\_2.weight\_hh}
        \item \texttt{gru\_2.bias\_hh}
    \end{itemize}
    \item LSTM
    \begin{itemize}[wide=10pt, leftmargin=*, topsep=0pt]
        \item \texttt{lstm\_1.weight\_ih}
        \item \texttt{lstm\_1.weight\_hh}
        \item \texttt{lstm\_1.bias\_ih}
        \item \texttt{lstm\_1.weight\_hh}
        \item \texttt{lstm\_1.bias\_hh}
        \item \texttt{lstm\_2.weight\_ih}
        \item \texttt{lstm\_2.weight\_hh}
        \item \texttt{lstm\_2.bias\_ih}
        \item \texttt{lstm\_2.weight\_hh}
        \item \texttt{lstm\_2.bias\_hh}
    \end{itemize}
    \item Transformer
    \begin{itemize}[wide=10pt, leftmargin=*, topsep=0pt]
        \item \texttt{encoder\_input\_layer.weight}
        \item \texttt{encoder\_input\_layer.bias}
        \item \texttt{decoder\_input\_layer.weight}
        \item \texttt{decoder\_input\_layer.bias}
        \item \texttt{linear\_mapping.weight}
        \item \texttt{linear\_mapping.bias}
    \end{itemize}
\end{itemize}

\section{Additional Literature Review}
\subsection{Knowledge-enhanced learning algorithms}
The concept of incorporating rigorous, formal logic specifications to enhance machine learning models is rooted in the approach of integrating prior knowledge into machine learning, which is also known as \textit{informed machine learning}. This approach has been extensively discussed in recent surveys such as Von Rueden et al.~\cite{von2021informed} and Karniadakis et al.~\cite{karniadakis2021physics}. By incorporating domain-specific auxiliary knowledge and constraints, informed machine learning methods aim to improve the performance of data-driven predictive models from many aspects. 
For instance, Muralidhar et al.\cite{muralidhar2018incorporating} propose a domain knowledge-aware loss function that enables learning models to capture the quantitative range and trend of the data more effectively. This approach enhances the model's ability to represent increasing or decreasing patterns within the data. Similarly, Diligenti et al.\cite{diligenti2017integrating} utilize first-order logic to express property constraints and introduce a logic-based loss function to penalize any violations of these constraints during model training.
In general, these approaches enable machine learning models to incorporate meaningful constraints and capture essential characteristics of the data. 

In addition to improvements in loss functions and optimization objectives, other works have explored various strategies to incorporate prior knowledge into neural networks by designing knowledge-injected network structures. These approaches aim to encode inherent model properties and improve the model's ability to capture relevant patterns and characteristics of the data.
For example, Jia et al.~\cite{jia2021physics} enhance the standard recurrent neural network (RNN) block by introducing two additional gates. These gates allow the network to extract seasonal and yearly patterns from lake temperature data, enabling the model to better capture and represent long-term temporal dependencies.
Another popular strategy is to leverage different variations of machine learning frameworks to integrate prior knowledge effectively. For instance, Ma et al.~\cite{ma2020stlnet} propose a teacher-student RNN model for cyber-physical systems. In this approach, the teacher generates traces that conform to the desired model properties, while the student network observes and mimics the teacher's behavior, thereby learning to capture the desired system dynamics. Furthermore, Hu et al.~\cite{hu2020harnessing} propose a method that involves projecting neural network predictions onto a logic-constrained subspace. By enforcing adherence to logic rules, the model ensures that its predictions satisfy the specified logical constraints. 

However, applying the previous methods directly to Federated Learning (FL) tasks poses a critical challenge. It is impractical to expect meaningful and accurate data properties to be defined for a large number of clients. Additionally, involving domain experts to manually examine client datasets introduces potential security risks to the existing system. 
Unlike previous approaches that rely on pre-defined data properties, our proposed approach, FedSTL, overcomes these challenges by dynamically exploring meaningful data properties locally. This eliminates the need for manual input and mitigates privacy concerns. 
Moreover, previous approaches are limited in terms of the number of prior knowledge that can be enforced on predictive models. For example, the work by Jia et al. is confined to a specific scenario of lake temperature prediction. In contrast, our approach takes advantage of specification mining techniques, enabling the exploration of a wide variety of logic reasoning specifications. This flexibility allows us to apply the method to diverse FL scenarios and leverage the power of logic-based reasoning.

\subsection{Personalized federated learning}
FL is a privacy-preserving learning paradigm that allows local, decentralized devices to collectively contribute to a shared global model without sharing their sensitive local databases. 
More concretely, the objective of FL is to find an optimal model that performs well across all clients~\cite{mcmahan2017communication}.
While the optimization of the shared global model is central to FL and its variants~\cite{li2020federated}, applying FL to real-world scenarios introduces challenges such as limited data, statistically heterogeneous datasets, and insufficient computing power, which can significantly impact the performance of the shared model on individual devices~\cite{nguyen2021federated, li2020review, yang2019federated}. 

In contrast to traditional FL frameworks, personalized FL prioritizes the training of local models tailored to individual clients, rather than relying on a single global model. Recent work has shifted the learning objective to focus on training local models while leveraging the assistance or initialization of a pre-trained global model~\cite{tan2022towards}. To effectively adapt the shared global model to client-specific tasks, several approaches have been developed, including model regularization~\cite{li2021ditto}, fine-tuning~\cite{NEURIPS2022_3cc03e19}, client clustering~\cite{ghosh2020efficient}, and transfer learning~\cite{mansour2020three}. These techniques aim to enhance the performance and customization of the FL framework to meet diverse client requirements and account for variations in data characteristics.

More specifically, Ghosh et al.~\cite{ghosh2020efficient} propose a clustering algorithm that addresses conflicting learning tasks among FL clients by re-identifying client-cluster affiliation in each communication round using client empirical losses. In contrast, our solution assigns FL clients to clusters based on the alignment of the induced logic reasoning property. Additionally, other approaches such as Collins et al.\cite{collins2021exploiting} and Arivazhagan et al.~\cite{arivazhagan2019federated} leverage model layer customization and train personalized layers to alleviate data heterogeneity

The FedSTL framework presented in this paper represents a significant advancement in both FL client clustering and layer-wise customization techniques. 
FedSTL distinguishes itself by employing a hierarchical approach, enhancing both client-level properties in customized layers and cluster-level properties in cluster-aggregated layers. This novel combination allows for improved performance and more effective utilization of both local and global information.

\section{Limitations}
In this section, we will discuss the limitations of our study and provide suggestions for future work. 
First, it is important to acknowledge that our multivariate evaluation task is conducted in a simulated environment with selected traffic elements. To fully assess the applicability of our approach, further investigations are needed in real-world multivariate Smart Transportation and Smart City use cases. 
Additionally, our current experiments do not consider public transportation modes with designated routes and schedules, such as light rail, local shuttles, and buses. Including these aspects in future experiments would provide a more comprehensive evaluation of our method.
Furthermore, it would be valuable to assess the generalizability of our proposed method to other IoT applications mentioned in Appendix~\ref{app:b}. Exploring the effectiveness of our approach in different domains would contribute to its broader applicability.
Moreover, in constructing the FL client databases for traffic prediction, we considered all types of sensor functional classifications. Future work could focus on developing more advanced methods to construct fine-grained client datasets and specifications, which could lead to more accurate and personalized predictions.
Finally, it is worth noting that our evaluation involved a limited number of FL clients (100 clients). Expanding the scale of the evaluation with a larger number of clients would provide a more comprehensive understanding of the performance and scalability of our approach.

% \section*{Potential Negative Social Impact}
% To the best of our knowledge, there is no potential negative social impact. 

\end{document}